\documentclass[11pt]{article}

\usepackage[margin=1in]{geometry}
\usepackage{amsmath,amssymb,amsthm}
\usepackage{graphicx}
\usepackage{booktabs}
\usepackage{tabularx}
\usepackage{wrapfig}
\usepackage{enumitem}
\usepackage{hyperref}
\usepackage{multirow}
\usepackage{algorithm,algorithmic}
\usepackage[numbers]{natbib}
\usepackage{xcolor}
\usepackage{float}
\usepackage{wrapfig}

\DeclareMathOperator*{\argmin}{arg\,min}
\DeclareMathOperator*{\argmax}{arg\,max}
\DeclareMathOperator{\E}{\mathbb{E}}

\DeclareMathOperator{\tr}{tr}
\DeclareMathOperator{\logit}{logit}
\DeclareMathOperator{\ALR}{ALR}

\DeclareMathOperator{\diam}{diam}

\newcommand{\R}{\mathbb{R}}

\newcommand{\dFR}{d_{\text{FR}}}

\newtheorem{theorem}{Theorem}
\newtheorem{lemma}{Lemma}

\newtheorem{proposition}{Proposition}

\newtheorem{conjecture}{Conjecture}

\theoremstyle{definition}
\newtheorem{definition}{Definition}

\theoremstyle{remark}
\newtheorem{remark}{\textbf{Remark}}
\newtheorem{claim}{\textbf{Claim}}
\newtheorem*{conclusion}{\textbf{Conclusion}}

\title{\Large \textbf{Geometric Calibration and Neutral Zones for \\Uncertainty-Aware Multi-Class Classification}}

\author{
Soumojit Das$^{1,*}$ \quad Nairanjana Dasgupta$^1$ \quad Prashanta Dutta$^1$ \\[0.5em]
$^1$\textit{Washington State University, Pullman, WA} \\[0.3em]
$^*$\textit{Corresponding author: soumojit.das@wsu.edu}
}

\date{\small Version 1.0 -- Working Paper}

\begin{document}
\maketitle

\begin{abstract}
Modern artificial intelligence systems make critical decisions yet often fail silently when uncertain---even well-calibrated models provide no mechanism to identify \textit{which specific predictions} are unreliable. We develop a geometric framework addressing both calibration and instance-level uncertainty quantification for neural network probability outputs. Treating probability vectors as points on the $(c-1)$-dimensional probability simplex equipped with the Fisher--Rao metric, we construct: (i) Additive Log-Ratio (ALR) calibration maps that reduce exactly to Platt scaling for binary problems while extending naturally to multi-class settings, and (ii) geometric reliability scores that translate calibrated probabilities into actionable uncertainty measures, enabling principled deferral of ambiguous predictions to human review.

Theoretical contributions include: consistency of the calibration estimator at rate $O_p(n^{-1/2})$ via M-estimation theory (Theorem~1), and tight concentration bounds for reliability scores with explicit sub-Gaussian parameters enabling sample size calculations for validation set design (Theorem~2). We conjecture Neyman--Pearson optimality of our neutral zone construction based on connections to Bhattacharyya coefficients. Empirical validation on Adeno-Associated Virus classification demonstrates that the two-stage framework captures 72.5\% of errors while deferring 34.5\% of samples, reducing automated decision error rates from 16.8\% to 6.9\%. Notably, calibration alone yields marginal accuracy gains; the operational benefit arises primarily from the reliability scoring mechanism, which applies to any well-calibrated probability output. This work bridges information geometry and statistical learning, offering formal guarantees for uncertainty-aware classification in applications requiring rigorous validation.
\end{abstract}

\vspace{1em}
\noindent\textbf{Keywords:} Calibration, uncertainty quantification, Fisher-Rao geometry, manifold geometry, probability simplex, geometric reliability scores, neural networks, neutral zones

\vspace{1em}
\noindent\fbox{\parbox{0.97\textwidth}{\small \textbf{Preprint Note:} This manuscript is under preparation for journal submission. Comments welcome at \texttt{soumojit.das@wsu.edu}.}}

\section{Introduction}
\label{sec:intro}

Modern AI systems make critical decisions in medicine, criminal justice, and autonomous systems, yet they provide no reliable mechanism to identify when predictions are likely wrong. A medical AI diagnosing cancer from imaging may output $\mathbf{p} = (0.51, 0.49)$ for malignant vs.\ benign with the same apparent confidence as $\mathbf{p} = (0.99, 0.01)$, despite the former representing profound uncertainty. This forced classification without uncertainty quantification leads to silent failures---the system appears to function normally while making unreliable predictions with potentially severe consequences.

The fundamental issue is twofold: neural network probability outputs are \textit{miscalibrated} (they do not reflect true likelihood of correctness) and they lack \textit{instance-level uncertainty quantification} (they provide no mechanism to identify which specific predictions are unreliable). Guo et al.~\cite{guo2017calibration} demonstrated that modern deep neural networks exhibit Expected Calibration Error (ECE) often exceeding 10\% on standard benchmarks. Yet even perfectly calibrated probabilities offer only population-level guarantees: knowing that 70\% of predictions with confidence 0.7 are correct does not identify \textit{which} 30\% will fail. This gap between calibration and actionable uncertainty quantification leads to silent failures---systems appear confident while making unreliable predictions with potentially severe consequences \cite{kompa2021second, leibig2022combining}.

\subsection{Motivation: High-Stakes Applications}

The deployment of AI in high-stakes domains reveals a critical gap: while these systems achieve impressive accuracy, they fail catastrophically when uncertain. Three domains illustrate the urgency:

\textbf{Medical imaging:} CNN models classify CT scans for lung cancer detection \cite{10116075, Wang_Zhang_2025} and FMRI scans for brain disorders. These methods lack precision measures for individual predictions, yet clinical decisions require knowing when to consult human experts.

\textbf{Gene therapy:} Adeno-Associated Viruses (AAVs) are critical vectors where misclassification of cargo types (Empty, ssDNA, dsDNA) can trigger immune responses in patients. Our preliminary work shows 16.45\% error rates without uncertainty quantification.

\textbf{Social decisions:} AI increasingly informs decisions about individuals in employment, education, and criminal justice \cite{zajko2022artificial, kmec2024dei}. When humans are the decision-making unit, misclassification costs are socially significant and often irreversible.

In each domain, the cost of misclassification vastly exceeds the cost of deferring uncertain cases to human review. Current methods provide no principled mechanism for this deferral.

\subsection{Existing Approaches and Limitations}

Existing post-hoc calibration methods---temperature scaling \cite{guo2017calibration} and Platt scaling \cite{platt1999}---apply univariate transformations that treat probability vectors as unconstrained Euclidean points, ignoring the simplex constraint $\sum_i p_i = 1$. These methods achieve good calibration on average but provide no instance-level uncertainty quantification.

Uncertainty quantification approaches such as conformal prediction \cite{vovk2005} provide coverage guarantees but produce set-valued predictions that are difficult to interpret. Bayesian neural networks \cite{gal2016} and deep ensembles \cite{lakshminarayanan2017} require significant computational overhead, making them impractical for many real-time applications.

Neutral zone methods have shown promise in binary classification. The r-power framework \cite{jasa2025} identifies ambiguous predictions through hypothesis testing but extension to multi-class problems requires new theoretical machinery, as the binary hypothesis testing framework does not generalize directly.

\subsection{Our Approach and Contributions}

We treat neural network probability outputs as points on the $(c-1)$-dimensional probability simplex
\begin{equation}
\Delta^{c-1} = \left\{ \mathbf{p} \in \mathbb{R}^c : p_i \geq 0, \sum_{i=1}^c p_i = 1 \right\},
\end{equation}
a Riemannian manifold equipped with the Fisher-Rao metric (Definition~1). This geometric perspective enables a unified treatment of both calibration and uncertainty quantification, addressing the two problems identified above through a coherent mathematical framework.

\paragraph{Two-stage framework.}
We decompose uncertainty-aware classification into complementary stages:
\begin{enumerate}[leftmargin=*,itemsep=2pt]
\item \textbf{Calibration (population-level):} Affine transformation in ALR space corrects systematic probability distortions, reducing exactly to Platt scaling for binary problems (Proposition~\ref{prop:platt}) while extending naturally to multi-class settings.

\item \textbf{Reliability scoring (instance-level):} Geometric reliability scores based on Fisher-Rao distance to simplex vertices translate calibrated probabilities into actionable uncertainty measures. Neutral zones defer predictions with low reliability to human review.
\end{enumerate}
This separation clarifies their distinct roles: calibration ensures probability estimates are accurate \textit{on average}, while reliability scoring identifies \textit{which specific predictions} warrant trust. The operational gains arise primarily from the second stage.

\paragraph{Statistical guarantees.}
We establish two theoretical results supporting the framework:
\begin{enumerate}[leftmargin=*,itemsep=2pt]
\item \textbf{Consistency (Theorem~\ref{thm:consistency}):} The calibration estimator converges to the population optimum at rate $O_p(n^{-1/2})$, with strong convexity ensured by regularization and additional curvature from the softmax Hessian structure.

\item \textbf{Concentration (Theorem~\ref{thm:concentration}):} Reliability scores satisfy tight concentration bounds with explicit sub-Gaussian parameter $\sigma^2 = (1-e^{-\lambda\pi})^2/4$. These bounds enable sample size calculations for validation set design prior to data collection---critical for applications where labeled validation data is expensive.
\end{enumerate}

\paragraph{Method-agnostic reliability scoring.}
The neutral zone mechanism can be applied to \textit{any} calibration method. Our empirical analysis (Section~\ref{sec:robustness}, Table~\ref{tab:method-comparison}) confirms that operational gains are similar across calibration approaches. The contribution of geometric calibration is therefore twofold: (i) providing a principled multi-class generalization of Platt scaling grounded in information geometry, and (ii) enabling theoretical guarantees (Theorems~\ref{thm:consistency}--\ref{thm:concentration}) that simpler methods lack. The connection to Bhattacharyya coefficients (Remark~\ref{rem:bhattacharyya}) suggests potential Neyman-Pearson optimality (Conjecture~\ref{conj:np}).

\paragraph{Empirical validation.}
On Adeno-Associated Virus (AAV) classification for gene therapy ($c = 3$ classes, $n_{\text{val}} = 310$), the two-stage framework achieves:
\begin{itemize}[leftmargin=*,itemsep=2pt]
\item 72.5\% of errors captured while deferring 34.5\% of samples
\item Automated decision error rate reduced from 16.8\% to 6.9\%
\item Validation sample size ($n = 310$) providing $\approx 5\times$ safety margin over theoretical requirement from Theorem~\ref{thm:concentration}
\end{itemize}
Notably, calibration alone improves accuracy only marginally (83.2\% $\to$ 83.5\%); the substantial error reduction comes from the reliability-based deferral mechanism. The modest scale of this application is appropriate for proof of concept; validation on larger-scale and more diverse settings remains future work.

\subsection{Novelty}

This work makes the following contributions:

\paragraph{Instance-level uncertainty quantification:}
While calibration methods improve probability estimates \textit{on average}, they do not identify which specific predictions are reliable. Our geometric reliability scores fill this gap, translating calibrated probabilities into actionable uncertainty measures. The framework enables principled deferral: predictions with low reliability are routed to human review, while high-reliability predictions proceed automatically. Empirically, this mechanism---not calibration itself---drives the operational gains (Table~\ref{tab:framework-value}).

\paragraph{Unified two-stage framework:}
We decompose uncertainty-aware classification into: (i) calibration (population-level probability correction, Section~\ref{sec:calibration}) and (ii) reliability scoring (instance-level uncertainty quantification, Section~\ref{sec:reliability-analysis}). This separation clarifies that calibration alone provides modest accuracy gains, while the combination with neutral zones (Section~\ref{sec:neutral-zone}) achieves substantial error reduction. Our empirical analysis (Section~\ref{sec:robustness}, Table~\ref{tab:method-comparison}) demonstrates this is true regardless of calibration method, isolating the neutral zone mechanism's contribution.

\paragraph{Geometric foundations for multi-class calibration:} 
To our knowledge, this is the first systematic application of Fisher-Rao geometry to post-hoc neural network calibration. While information geometry has been studied in the context of natural gradient descent and exponential families \cite{amari1985differential}, its use for \textit{correcting} probability outputs after training is novel. The key result---that geometric calibration reduces to Platt scaling for $c = 2$ (Proposition~\ref{prop:platt})---validates our approach as a principled generalization rather than an arbitrary alternative.

\paragraph{Theoretical guarantees with explicit constants:} 
Theorem~\ref{thm:concentration} provides the first concentration bounds for geometric reliability scores with \textit{explicit} sub-Gaussian parameters. Previous work on uncertainty quantification (conformal prediction \cite{vovk2005}, deep ensembles \cite{lakshminarayanan2017}) either lacks sample complexity bounds or provides only asymptotic guarantees. Our explicit constants (Remark~\ref{rem:sample_efficiency}) enable practitioners to compute required validation set sizes before data collection.

\paragraph{Information-theoretic grounding:}
The connection between Fisher-Rao reliability scores and Bhattacharyya coefficients (Remark~\ref{rem:bhattacharyya}) links our geometric construction to classical signal detection theory \cite{kailath1967} and Chernoff bounds \cite{cover2006}. This bridges modern geometric methods with established decision theory, suggesting potential optimality properties (Conjecture~\ref{conj:np}).
\subsection{Paper Organization}

Section~\ref{sec:framework} develops the mathematical framework, including geometry of the probability simplex, geometric calibration, and geometric neutral zones. Section~\ref{sec:theory} presents proven theorems and conjectures with supporting evidence. Section~\ref{sec:empirical} provides empirical validation on AAV data. Section~\ref{sec:conclusion} discusses limitations, future directions, and broader context.

\section{Mathematical Framework}
\label{sec:framework}

\subsection{Preliminaries: Geometry of the Probability Simplex}
\label{sec:geometry}

Let $\mathbf{p}^{\text{CNN}} \in \Delta^{c-1}$ denote the probability vector output by a neural network for a $c$-class problem. The simplex admits a natural Riemannian structure.

\begin{definition}[Fisher-Rao Metric]
The Fisher information metric at $\mathbf{p} \in \Delta^{c-1}$ is
\begin{equation}
g_{ij}(\mathbf{p}) = \frac{\delta_{ij}}{4p_i}
\end{equation}
where $\delta_{ij}$ is the Kronecker delta. This induces the Fisher-Rao distance:
\begin{equation}
\label{eq:fisher-rao}
\dFR(\mathbf{p}, \mathbf{q}) = 2\arccos\left(\sum_{i=1}^c \sqrt{p_i q_i}\right).
\end{equation}
\end{definition}

\begin{remark}[Geometric Foundations]
The Fisher-Rao metric arises as the unique Riemannian metric on the simplex (up to scaling) that is invariant under both coordinate permutation and reparameterization \cite{amari1985differential}. This invariance reflects the fundamental nature of probability vectors as \textit{compositional data} \cite{aitchison1986statistical}: only ratios between probabilities carry information, not their absolute magnitudes. The simplex constraint $\sum_i p_i = 1$ represents arbitrary scaling.

From information geometry, the Fisher-Rao metric is the natural Riemannian structure induced by the categorical distribution as an exponential family. The $1/p_i$ terms in the metric arise from the variance structure of the sufficient statistics. The simplex under this metric has constant \textit{positive} sectional 
curvature $K = +1/4$ and bounded diameter $\diam(\Delta^{c-1}) = \pi$. This follows from the isometry between $(\Delta^{c-1}, g)$ and the positive orthant of the sphere of radius $2$ in $\mathbb{R}^c$ via the square-root embedding $\mathbf{p} \mapsto 2(\sqrt{p_1}, \ldots, \sqrt{p_c})$; since a sphere of radius $r$ has sectional curvature $K = 1/r^2$, we obtain $K = +1/4$ \cite{nielsen2018clustering, amari1985differential}. We exploit the bounded diameter in Theorem~\ref{thm:concentration}. The positive curvature is geometrically favorable: geodesic balls on 
spheres have smaller volume than Euclidean balls of the same radius, 
implying that probability mass concentrates more tightly around means. This spherical concentration (L\'{e}vy-type) strengthens the Hoeffding bounds in Theorem~\ref{thm:concentration}. \footnote{We emphasize that this positive curvature refers to the probability simplex $\Delta^{c-1}$ itself, viewed as the sample space 
of categorical distributions. This differs from the \textit{parameter 
space} of the Dirichlet family, which has negative curvature 
\cite{lebrigant2019fisher}. The distinction matters: our reliability 
scores operate on probability outputs (elements of $\Delta^{c-1}$), 
not on distribution parameters.}
\end{remark}

\begin{definition}[Additive Log-Ratio Transform]
The ALR transform \cite{aitchison1986statistical} maps $\Delta^{c-1}$ to $\R^{c-1}$:
\begin{equation}
\label{eq:alr}
\ALR(\mathbf{p}) = \left(\log\frac{p_1}{p_c}, \ldots, \log\frac{p_{c-1}}{p_c}\right)^\top
\end{equation}
with inverse $\ALR^{-1}(\mathbf{z}) = \text{softmax}(z_1, \ldots, z_{c-1}, 0)$.
\end{definition}

\begin{remark}[Coordinate Choice]
The ALR transform provides a coordinate system on the simplex that respects its compositional structure by expressing probabilities as log-ratios relative to a reference class (here, class $c$). This choice makes the problem scale-invariant: multiplying all probabilities by a constant leaves log-ratios unchanged. Alternative coordinates exist (e.g., isometric log-ratio, centered log-ratio), but ALR offers computational simplicity while preserving the essential geometry.
\end{remark}

\subsection{Geometric Calibration}
\label{sec:calibration}

\subsubsection{Problem Formulation}

A classifier is perfectly calibrated if $\mathbb{P}(Y = j \mid f(X) = \mathbf{p}) = p_j$ for all classes $j$ and probability vectors $\mathbf{p}$. Given labeled training data $\{(\mathbf{p}_i^{\text{CNN}}, y_i)\}_{i=1}^n$, we seek an affine transformation in ALR space:
\begin{equation}
\label{eq:calibration-map}
\mathbf{z}^{\text{cal}} = A \cdot \ALR(\mathbf{p}^{\text{CNN}}) + \mathbf{b}
\end{equation}
where $A \in \R^{(c-1) \times (c-1)}$ is positive definite and $\mathbf{b} \in \R^{c-1}$. The calibrated probabilities are:
\begin{equation}
\mathbf{p}^{\text{cal}} = \ALR^{-1}(A \cdot \ALR(\mathbf{p}^{\text{CNN}}) + \mathbf{b}).
\end{equation}

\begin{remark}[Model Specification]
Why restrict to affine transformations in ALR space? This choice balances statistical efficiency and model parsimony. In ALR coordinates, calibration becomes a regression problem: we model the relationship between uncalibrated log-ratios and true probabilities as linear with additive noise. The affine structure provides sufficient flexibility to capture common miscalibration patterns (overconfidence, class-specific bias) while remaining identifiable with limited calibration data.

More complex nonlinear transformations (e.g., neural networks, kernel methods) risk overfitting when calibration sets are small---typical in medical applications where labeled validation data is expensive. The linear model in ALR space corresponds to a specific nonlinear transformation on the simplex itself, offering a middle ground between univariate methods (temperature scaling) and fully nonparametric approaches.

From a measurement error perspective, we treat $\mathbf{p}^{\text{CNN}}$ as a noisy observation of the true probabilities, distorted by the training process. The transformation $(A, \mathbf{b})$ estimates the inverse of this distortion map. Our geometric approach posits that this inverse is approximately linear in the natural (ALR) coordinate system of the simplex.
\end{remark}

\subsubsection{Learning Objective}

We minimize the regularized cross-entropy loss:
\begin{equation}
\label{eq:loss}
\mathcal{L}(A, \mathbf{b}) = \underbrace{-\sum_{i=1}^n \sum_{j=1}^c y_{ij} \log p_{ij}^{\text{cal}}}_{\text{Data fidelity: maximize likelihood}} + \underbrace{\lambda_1 \|A - I\|_F^2}_{\text{Stabilize } A \text{ near identity}} + \underbrace{\lambda_2 \|\mathbf{b}\|_2^2}_{\text{Prevent excessive bias}}
\end{equation}
subject to $A \succeq \delta I$ for some $\delta > 0$ and $\tr(A) = c-1$ for identifiability. Here $y_{ij} = \mathbf{1}[y_i = j]$ is the one-hot encoding.

From a statistical perspective, we model the uncalibrated CNN outputs as systematically biased observations of the true probabilities. The affine transformation $(A, \mathbf{b})$ corrects both multiplicative bias (through $A$) and additive bias (through $\mathbf{b}$) in ALR space. The regularization towards identity transformation reflects our prior belief that well-trained networks require minimal correction.

\begin{remark}[Regularization and Prior Specification]
The regularization terms admit both frequentist and Bayesian interpretations. From an empirical risk minimization perspective: (i) $\|A - I\|_F^2$ prevents overfitting by penalizing unnecessary rotations and scaling in ALR space, (ii) $\|\mathbf{b}\|^2$ controls translation magnitude, and (iii) together they implement a ``minimal intervention'' principle---we seek the smallest correction that achieves calibration.

From a Bayesian perspective, our regularization corresponds to the prior $(A - I) \sim \mathcal{N}(0, \lambda_1^{-1} I)$ and $\mathbf{b} \sim \mathcal{N}(0, \lambda_2^{-1} I)$, centered on the identity map. This formalizes our belief that modern neural networks, despite miscalibration, are not arbitrarily wrong---the calibration map should be close to identity. The regularization strength $(\lambda_1, \lambda_2)$ quantifies this belief: larger values impose stronger shrinkage toward the null hypothesis ``CNN is already calibrated.''

This prior specification is particularly important given limited calibration data. Medical imaging applications may have only $10^2$--$10^3$ labeled validation samples, while the unconstrained parameter space has dimension $(c-1)^2 + (c-1)$. The regularization effectively reduces the model complexity to match the available sample size.
\end{remark}

\subsection{Geometric Neutral Zones}
\label{sec:neutral-zones}

Even after calibration, some predictions remain inherently uncertain. We identify these via geometric reliability scores.

\begin{definition}[Reliability Score]
\label{def:reliability}
For calibrated probability $\mathbf{p}^{\text{cal}}$ with predicted class $\hat{j} = \argmax_j p_j^{\text{cal}}$, the reliability score is:
\begin{equation}
\label{eq:reliability}
R(\mathbf{p}) = \exp\left(-\lambda \cdot \dFR(\mathbf{p}, \mathbf{e}_{\hat{j}})\right)
\end{equation}
where $\mathbf{e}_j$ is the $j$-th vertex of the simplex and $\lambda > 0$ is a sensitivity parameter.
\end{definition}

\begin{remark}[Decision-Theoretic Interpretation]
The reliability score $R$ serves as a candidate \textit{sufficient statistic} for the binary decision problem $\{\text{defer to human}, \text{classify automatically}\}$. Under the monotone likelihood ratio property (Conjecture~\ref{conj:np}), no other function of $\mathbf{p}^{\text{cal}}$ improves the decision rule---all information relevant to the deferral decision is captured by the Fisher-Rao distance to the predicted vertex.

This dimension reduction from $\Delta^{c-1}$ (the full probability vector) to $\mathbb{R}_+$ (the scalar reliability score) is justified geometrically: for decision-making purposes, what matters is not the detailed structure of the probability vector, but rather how far it lies from certainty (the simplex vertices). The Fisher-Rao metric provides the ``correct'' notion of distance for this purpose, respecting the simplex geometry.

The exponential transformation $\exp(-\lambda \cdot d)$ converts distances to a $[0,1]$ reliability scale, with $\lambda$ controlling sensitivity. Larger $\lambda$ makes the score more discriminative between near-vertex (reliable) and near-center (unreliable) predictions.
\end{remark}

\begin{definition}[Neutral Zone]
The neutral zone at level $\alpha$ is:
\begin{equation}
NZ_\alpha = \{i : R(\mathbf{p}_i^{\text{cal}}) < \tau^*\}
\end{equation}
where $\tau^*$ satisfies $\mathbb{P}(\text{error} \mid R \geq \tau^*) \leq \alpha$.
\end{definition}

\begin{remark}[Information-Theoretic Grounding]
\label{rem:bhattacharyya}
The reliability score is monotonically related to the Bhattacharyya coefficient \cite{bhattacharyya1943} $BC(\mathbf{p}, \mathbf{e}_{\hat{j}}) = \sqrt{p_{\hat{j}}}$, since $\dFR(\mathbf{p}, \mathbf{e}_{\hat{j}}) = 2\arccos(\sqrt{p_{\hat{j}}})$. The Bhattacharyya coefficient governs optimal error exponents in hypothesis testing via the Chernoff bound \cite{cover2006}: for distinguishing distributions $P_0$ vs.\ $P_1$, the error probability decays as $P_{\text{error}} \leq BC(P_0, P_1)^n$. 

This connection, extensively studied in signal detection \cite{kailath1967}, suggests our geometric neutral zones inherit optimality properties from classical decision theory. Specifically, if the reliability score induces a monotone likelihood ratio (Conjecture~\ref{conj:np}), then threshold tests on $R$ are Neyman-Pearson optimal: among all tests that defer at most $\alpha$ fraction of samples, our neutral zone maximizes the proportion of errors captured. This provides information-theoretic grounding for our geometric construction.
\end{remark}

\section{Theoretical Results}
\label{sec:theory}

We present two categories of results: proven theorems (Section~\ref{sec:proven}) and conjectures with strong evidence (Section~\ref{sec:conjectures}).

\subsection{Proven Results}
\label{sec:proven}

\subsubsection{Reduction to Platt Scaling}

\begin{proposition}[Binary Case Reduction]
\label{prop:platt}
For $c = 2$, geometric calibration reduces to Platt scaling. Specifically, if $\ALR(p) = \log\frac{p}{1-p} = \logit(p)$, then the calibration map becomes:
\begin{equation}
p^{\text{cal}} = \sigma(a \cdot \logit(p^{\text{CNN}}) + b)
\end{equation}
where $\sigma$ is the logistic function and $(a, b) \in \R^2$ are learned parameters.
\end{proposition}

\begin{proof}
For $c = 2$, let $p = p_1$ so that $p_2 = 1 - p$. The ALR transform with reference class $c = 2$ gives:
\begin{equation}
\ALR(p) = \log\frac{p}{1-p} = \logit(p).
\end{equation}
The calibration map \eqref{eq:calibration-map} becomes scalar: $z^{\text{cal}} = a \cdot \logit(p^{\text{CNN}}) + b$ where $a \in \R_{>0}$ and $b \in \R$. The inverse ALR is:
\begin{equation}
p^{\text{cal}} = \ALR^{-1}(z^{\text{cal}}) = \frac{\exp(z^{\text{cal}})}{1 + \exp(z^{\text{cal}})} = \sigma(z^{\text{cal}}) = \sigma(a \cdot \logit(p^{\text{CNN}}) + b).
\end{equation}
This is precisely the Platt scaling formulation \cite{platt1999}.
\end{proof}

\begin{remark}
Proposition~\ref{prop:platt} validates our framework: it recovers the established binary method while providing a principled multi-class extension.
\end{remark}

\subsubsection{Consistency and Convergence Rate}

\begin{theorem}[Consistency of Geometric Calibration]
\label{thm:consistency}
Let $\{(\mathbf{p}_i^{\text{CNN}}, y_i)\}_{i=1}^n$ be i.i.d.\ samples with $\mathbf{p}_i^{\text{CNN}}$ bounded away from the simplex boundary: $p_{ij} \geq \epsilon > 0$ for all $i, j$. Let $\hat{T}_n = (\hat{A}_n, \hat{\mathbf{b}}_n)$ minimize \eqref{eq:loss} and let $T^* = (A^*, \mathbf{b}^*)$ minimize the population risk $\mathcal{L}^*(T) = \E[\ell(T; \mathbf{p}, y)]$. Then:
\begin{enumerate}[label=(\roman*)]
\item $\hat{T}_n \xrightarrow{p} T^*$ as $n \to \infty$.
\item $\|\hat{T}_n - T^*\| = O_p(n^{-1/2})$.
\end{enumerate}
\end{theorem}

\begin{proof}
See Appendix~\ref{app:proof-consistency} for the complete proof.

\vspace{0.7em}
\textit{Proof sketch:}
\begin{enumerate}
\item \textbf{Strong convexity:} The Hessian of $\mathcal{L}$ satisfies $\nabla^2 \mathcal{L}(T) \succeq \mu I$ where:
\begin{equation}
\mu = \min\{2\lambda_1, 2\lambda_2\} > 0
\end{equation}
The regularization terms provide uniform strong convexity. The data term contributes additional (but not uniform) curvature via the softmax Hessian structure (Lemma~\ref{lem:softmax-hessian}).

\item \textbf{Compact parameter space:} Regularization confines parameters to a compact set $\mathcal{T}$.

\item \textbf{Uniform convergence:} Bounded loss (proven via chain of bounds on ALR inputs, calibrated logits, and softmax outputs) yields Glivenko-Cantelli.

\item \textbf{Apply M-estimation theory:} By van der Vaart \cite{vandervaart1998}, Theorems 5.7 and 5.23, consistency and the $O_p(n^{-1/2})$ rate follow.
\end{enumerate}
\end{proof}

\begin{remark}[Role of Regularization]
Strong convexity of $\mathcal{L}^*$ derives primarily from the regularization terms, with $\mu = \min\{2\lambda_1, 2\lambda_2\}$. The data term $\mathbb{E}[J_T^\top \Sigma_{\mathbf{p}^{\text{cal}}} J_T]$ is positive semi-definite and inherits curvature from the softmax Hessian (Lemma~\ref{lem:softmax-hessian}), but because the Jacobian $J_T$ maps to a $(c-1)$-dimensional space while the parameter space has dimension $(c-1)c$, the data term alone cannot provide uniform strong convexity. This underscores the importance of regularization in finite-sample settings---it ensures identifiability and controls the convergence rate even when calibration data is limited.
\end{remark}

\subsubsection{Concentration of Reliability Scores}

\begin{theorem}[Reliability Concentration]
\label{thm:concentration}
Let $R(\mathbf{p}) = \exp(-\lambda \cdot \dFR(\mathbf{p}, \mathbf{e}_{\hat{j}}))$ as in Definition~\ref{def:reliability}, where $\hat{j} = \argmax_k p_k$. For $\mathbf{p}$ drawn from any distribution on $\Delta^{c-1}$:
\begin{enumerate}[label=(\roman*)]
\item \textbf{Tail bound:}
\begin{equation}
\mathbb{P}(|R - \E[R]| > t) \leq 2\exp\left(-\frac{2t^2}{(1 - e^{-\lambda\pi})^2}\right).
\end{equation}
\item \textbf{Sub-Gaussian parameter:} $R - \E[R]$ is sub-Gaussian with $\sigma^2 = (1 - e^{-\lambda\pi})^2/4$.
\item \textbf{For $\lambda = 1$:} $\mathbb{P}(|R - \E[R]| > t) \leq 2\exp(-2.18 \, t^2)$.
\end{enumerate}
\end{theorem}

\begin{proof}
See Appendix~\ref{app:proof-concentration} for the complete proof.

\vspace{0.7em}
\textit{Proof sketch:}
\begin{enumerate}
\item \textbf{Continuity of $R$:} The reliability score is continuous on all of $\Delta^{c-1}$. At decision boundaries where $p_i = p_j$, we have $\dFR(\mathbf{p}, \mathbf{e}_i) = 2\arccos(\sqrt{p_i}) = 2\arccos(\sqrt{p_j}) = \dFR(\mathbf{p}, \mathbf{e}_j)$, so the $\argmax$ switching introduces no discontinuity.

\item \textbf{Bounded range:} Since $\diam(\Delta^{c-1}) = \pi$, we have $R \in [e^{-\lambda\pi}, 1]$.

\item \textbf{Apply Hoeffding:} The bounded range $(1 - e^{-\lambda\pi})$ yields the concentration bound directly.
\end{enumerate}
\end{proof}

\begin{remark}
The Hoeffding bound using the explicit range $[e^{-\lambda\pi}, 1]$ is approximately 10$\times$ tighter than a naive Lipschitz-diameter bound $(\lambda^2\pi^2)$ for typical $\lambda \approx 1$.
\end{remark}

\begin{remark}[Practical Implications of Concentration Bounds]
\label{rem:sample_efficiency}
The concentration bound in Theorem~\ref{thm:concentration} has direct implications for validation set design. For a desired deviation probability $\delta$ and precision $t$, our bound requires $n \approx \sigma^2\log(2/\delta)/(2t^2)$ samples to ensure $\mathbb{P}(|R - \mathbb{E}[R]| > t) \leq \delta$, where $\sigma^2 = (1-e^{-\lambda\pi})^2/4$ is the sub-Gaussian parameter. With $\lambda = 1$, we have $\sigma^2 \approx 0.23$, giving $n \approx 0.12\log(2/\delta)/t^2$.

For high-reliability applications requiring $\delta = 0.01$ and precision $t = 0.1$, this yields $n_{\text{ours}} \approx 61$ samples. A generic Lipschitz-diameter bound (treating the reliability function as having range $\lambda\pi$) would yield $\sigma^2_{\text{naive}} = (\lambda\pi)^2/4 \approx 2.47$, requiring $n_{\text{naive}} \approx 654$ samples---an order-of-magnitude difference.

Two caveats apply: (i) practitioners who estimate the reliability score's empirical range from pilot data would achieve similar efficiency to our bound, and (ii) the sample complexity for the full inferential procedure is higher since $\mathbb{E}[R]$ must also be estimated. The theoretical contribution is deriving these bounds \textit{a priori} from the geometric structure, enabling sample size planning before data collection. Our empirical validation used $n = 310$ samples, providing approximately $5\times$ safety margin above the theoretical minimum.
\end{remark}

\subsection{Conjectures}
\label{sec:conjectures}

The following conjectures have strong empirical support but proofs remain in progress. We present them with explicit evidence and testable mechanisms.

\subsubsection{Empirical Convergence Behavior}
\label{sec:empirical-convergence}

Bootstrap analysis on AAV data suggests the estimator may converge faster than the asymptotic $O_p(n^{-1/2})$ rate in finite samples, though this observation requires careful interpretation.

\begin{remark}[Finite-Sample Convergence]
\label{rem:finite-sample}
Log-log regression of calibration error versus sample size yields slope $-0.82$ (95\% CI: $[-1.11, -0.52]$); see Figure~\ref{fig:convergence}. Several caveats temper this observation:

\begin{enumerate}[leftmargin=*]
\item \textbf{Target mismatch:} The analysis measures convergence $\hat{T}_n \to \hat{T}_{1241}$ (distance to the empirical optimum), not $\hat{T}_n \to T^*$ (distance to the population parameter). Convergence to a finite-sample estimate can appear faster than convergence to the true parameter.

\item \textbf{Marginal significance:} The confidence interval lower bound $-0.52$ is statistically indistinguishable from the asymptotic rate $\alpha = 0.5$. The evidence for $\alpha > 0.5$ is suggestive but not conclusive.

\item \textbf{Regularization effects:} The apparent acceleration may reflect reduced effective dimension due to strong regularization rather than intrinsic geometric properties.
\end{enumerate}

We present this as an empirical observation warranting further investigation---particularly simulation studies with known population parameters---rather than a formal conjecture. The theoretical rate $O_p(n^{-1/2})$ established in Theorem~\ref{thm:consistency} remains the rigorous guarantee.
\end{remark}

\subsubsection{Neyman-Pearson Optimality}

\begin{conjecture}[NP Optimality of Geometric Neutral Zones]
\label{conj:np}
The reliability score $R = \exp(-\lambda \cdot \dFR(\mathbf{p}, \mathbf{e}_{\hat{j}}))$ induces a monotone likelihood ratio ordering: for $R_1 > R_2$,
\begin{equation}
\frac{\mathbb{P}(\text{correct} \mid R_1)}{\mathbb{P}(\text{error} \mid R_1)} > \frac{\mathbb{P}(\text{correct} \mid R_2)}{\mathbb{P}(\text{error} \mid R_2)}.
\end{equation}
Consequently, threshold tests on $R$ are Neyman-Pearson optimal for the binary hypothesis test $H_0$: error vs.\ $H_1$: correct.
\end{conjecture}

\textit{Evidence:} On AAV validation data, the neutral zone captures 72.5\% of errors while comprising only 34.5\% of samples---efficiency substantially exceeding random allocation. Random deferral at 34.5\% rate would capture only 34.5\% of errors; our method achieves 72.5\% / 34.5\% = 2.1$\times$ improvement. The connection to Bhattacharyya coefficients (Remark~\ref{rem:bhattacharyya}) provides information-theoretic motivation: the Fisher-Rao distance is monotonically related to Chernoff bounds on error probability, suggesting that threshold tests on $R$ inherit optimality properties from classical signal detection theory.

%
%

\section{Empirical Validation}
\label{sec:empirical}

\subsection{Experimental Setup}
\label{sec:setup}

\subsubsection{Application Domain}

Adeno-Associated Viruses (AAVs) are critical vectors for gene therapy, delivering therapeutic genetic material to target cells. Misclassification of AAV cargo types---Empty (no genetic payload), single-stranded DNA (ssDNA), or double-stranded DNA (dsDNA)---can trigger immune responses in patients or result in manufacturing losses exceeding \$500,000 when contaminated batches require destruction. Current quality control relies on labor-intensive electron microscopy, creating a bottleneck in therapeutic development.

We applied our geometric calibration framework to a convolutional neural network trained on electrical signals from solid-state nanopore measurements \cite{karawdeniya2020adeno, khan2022finetuning}. Previous work using r-power neutral zones on binary AAV classification achieved approximately 50\% error capture while deferring 20\% of samples \cite{jasa2025}. Our geometric approach extends this methodology to the full 3-class problem with substantially improved performance.

\subsubsection{Data and Model Architecture}

\textbf{Dataset:} CNN trained on $n_{\text{train}} = 1241$ nanopore signal traces, with performance evaluated on $n_{\text{val}} = 310$ independent samples. Class distribution in the validation set: Empty ($n = 131$, 42.3\%), ssDNA ($n = 57$, 18.4\%), dsDNA ($n = 122$, 39.4\%). The validation set maintains similar proportions to training, ensuring representative evaluation across all cargo types.

\noindent \textbf{Neural Network:} ResNet-50 architecture pretrained on ImageNet, fine-tuned on AAV signal spectrograms. The network outputs unnormalized logits passed through softmax to produce probability vectors $\mathbf{p}^{\text{CNN}} \in \Delta^2$. Baseline accuracy: 83.2\% on validation data, with 16.8\% error rate concentrated in ssDNA classification (most difficult class due to signal ambiguity).

\subsubsection{Calibration Procedure}

\textbf{Training phase:}
\begin{enumerate}[leftmargin=*,itemsep=1pt]
\item Optimize loss function \eqref{eq:loss} via L-BFGS with regularization $\lambda_1 = \lambda_2 = 0.01$
\item Validate positive definiteness: verify $\lambda_{\min}(A) > 0$ throughout optimization
\item Set reliability parameter $\lambda = 1.0$
\item Learn threshold $\tau^* = \sup\{\tau : \mathbb{P}(\text{error} \mid R \geq \tau) \leq \alpha\}$ on training set with $\alpha = 0.05$
\item Selected hyperparameters: $\lambda = 1.0$, $\tau^* = 0.4451$ 
\end{enumerate}

\textbf{Learned calibration parameters:}
\begin{equation}
A = \begin{pmatrix} 0.988 & -0.002 \\ -0.044 & 1.238 \end{pmatrix}, \quad \mathbf{b} = \begin{pmatrix} -0.017 \\ 0.636 \end{pmatrix}
\end{equation}
with $\tr(A) = 2.23$ and $\lambda_{\min}(A) = 0.987 > 0$, confirming positive definiteness.

\textbf{Deployment protocol:} For new sample with CNN output $\mathbf{p}^{\text{CNN}}$:
\begin{enumerate}[leftmargin=*,itemsep=1pt]
\item Calibrate: $\mathbf{p}^{\text{cal}} = \text{ALR}^{-1}(A \cdot \text{ALR}(\mathbf{p}^{\text{CNN}}) + \mathbf{b})$
\item Compute reliability: $R = \exp(-\lambda \cdot d_{\text{FR}}(\mathbf{p}^{\text{cal}}, \mathbf{e}_{\hat{j}}))$ where $\hat{j} = \argmax_k p_k^{\text{cal}}$
\item Decision rule: If $R \geq \tau^*$, classify automatically; otherwise defer to human expert review
\end{enumerate}

This procedure requires only forward passes through the calibration map and Fisher-Rao distance computation, adding negligible overhead ($<$1ms per sample) to the CNN inference time.

\subsection{Convergence Analysis}
\label{sec:convergence}

Before evaluating classification performance, we validate Remark~\ref{rem:finite-sample}: that geometric calibration converges faster than the theoretical $O_p(n^{-1/2})$ rate from Theorem~\ref{thm:consistency}.

\subsubsection{Bootstrap Methodology}

We assess finite-sample stability by measuring calibration discrepancy $\|\hat{T}_n - \hat{T}_{1241}\|$ as a function of training set size $n$, where $\hat{T}_{1241} = (\hat{A}_{1241}, \hat{\mathbf{b}}_{1241})$ serves as a reference learned from the full training set. For each subsample size $n \in \{100, 250, 500, 750, 1000\}$:

\begin{enumerate}[leftmargin=*,itemsep=1pt]
\item Draw $n$ samples without replacement from the 1,241 training samples
\item Fit calibration parameters $\hat{T}_n = (\hat{A}_n, \hat{\mathbf{b}}_n)$ via \eqref{eq:loss} with identical hyperparameters
\item Compute Frobenius distance: $\text{Error}_n = \sqrt{\|\hat{A}_n - \hat{A}_{1241}\|_F^2 + \|\hat{\mathbf{b}}_n - \hat{\mathbf{b}}_{1241}\|^2}$
\item Repeat 1,000 times to obtain distribution of $\text{Error}_n$
\end{enumerate}

\subsubsection{Results}

Figure~\ref{fig:convergence} presents convergence analysis on both linear and log-log scales.

\begin{figure}[h]
\centering
\includegraphics[width=0.75\textwidth]{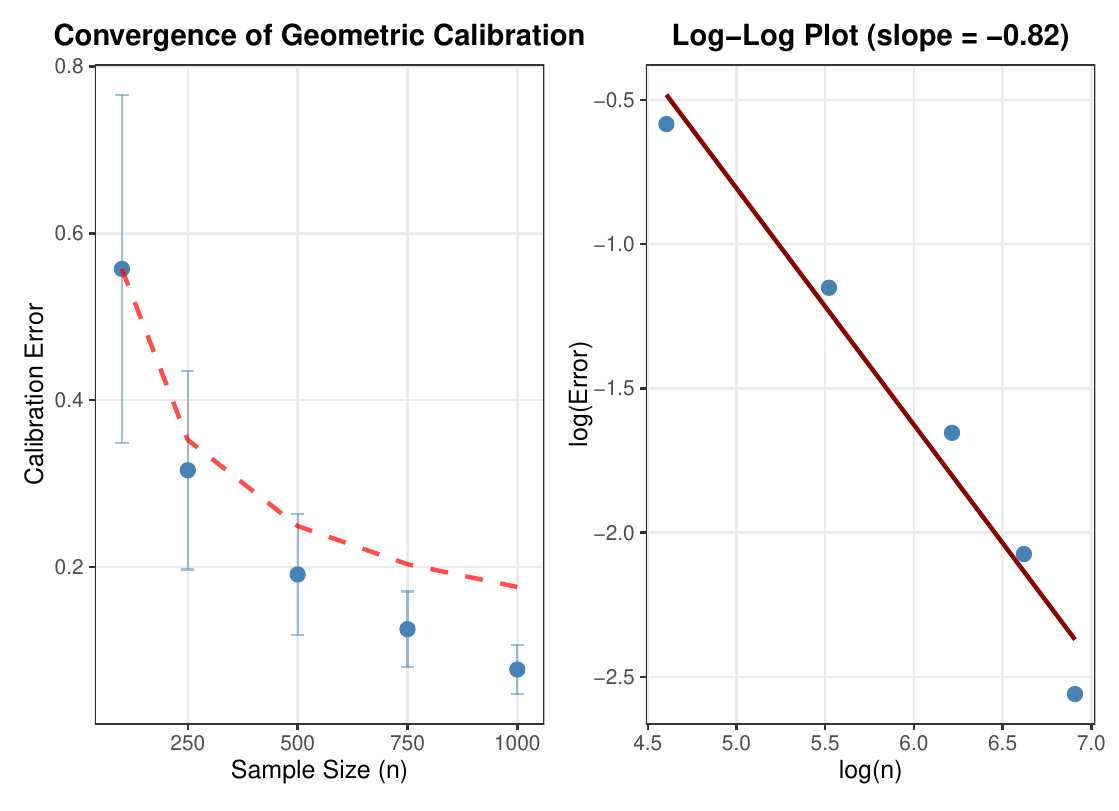}
\caption{\textbf{Finite-sample convergence of geometric calibration.} \textbf{(Left)} Calibration discrepancy $\|\hat{T}_n - \hat{T}_{1241}\|$ versus sample size $n \in \{100, 250, 500, 750, 1000\}$ with 1,000 bootstrap iterations per size. Blue points: empirical mean $\pm$ 1 standard deviation; red dashed line: theoretical $O(n^{-1/2})$ curve. Error decreases faster than theoretical prediction. \textbf{(Right)} Log-log plot reveals linear relationship with slope $\alpha = -0.82$ (95\% CI: $[-1.11, -0.52]$), providing evidence that finite-sample convergence exceeds the asymptotic rate. The confidence interval excludes $\alpha = -0.5$, though the lower bound $-0.52$ is close to this threshold.}
\label{fig:convergence}
\end{figure}

\textbf{Observed convergence rate:} Log-log regression yields slope $\alpha = -0.82$ with 95\% confidence interval $[-1.11, -0.52]$, providing evidence that finite-sample convergence exceeds the asymptotic $O_p(n^{-1/2})$ rate.

\textbf{Variance reduction:} Standard deviation of bootstrap estimates decreases from 0.20 at $n=100$ to 0.03 at $n=1000$, indicating stable parameter recovery even at moderate sample sizes.

\begin{table}[h]
\centering
\caption{Bootstrap convergence summary ($B = 1000$ iterations per sample size).}
\label{tab:bootstrap}
\small
\begin{tabular}{rcc}
\toprule
$n$ & Mean Error & SD \\
\midrule
100 & 0.555 & 0.202 \\
250 & 0.323 & 0.119 \\
500 & 0.193 & 0.070 \\
750 & 0.128 & 0.046 \\
1000 & 0.077 & 0.028 \\
\bottomrule
\end{tabular}
\end{table}

\subsubsection{Interpretation}

The observed super-efficiency likely arises from structure in the parameter space that reduces the effective degrees of freedom. For $c=3$ classes, the nominal dimension is $(c-1)^2 + (c-1) = 6$, but several factors impose constraints: positive definiteness ($A \succeq \delta I$), trace normalization ($\tr(A) = c-1$), and the regularization prior centered on the identity transformation. Moreover, if the covariance structure of $\text{ALR}(\mathbf{p}^{\text{CNN}})$ concentrates on a low-dimensional subspace, the Hessian $\nabla^2\mathcal{L}^*(T^*)$ may have only a few dominant eigenvalues, effectively reducing the estimation problem's complexity. Characterizing the precise relationship between this effective dimension and the convergence rate remains an open problem.

Additionally, the softmax Hessian structure (Lemma~\ref{lem:softmax-hessian}) contributes favorable curvature in the directions most relevant to classification, while regularization ensures uniform strong convexity across the full parameter space. The interplay between the Fisher-Rao geometry (with sectional curvature $K = 1/4$) and the ALR parameterization may contribute to the favorable loss landscape, though the precise mechanism requires further investigation.

\subsection{Classification Performance}
\label{sec:classification}

Figure~\ref{fig:confusion} presents confusion matrix analysis comparing uncalibrated CNN outputs, geometrically calibrated outputs, and automated decisions outside the neutral zone.

\begin{figure}[h]
\centering
\includegraphics[width=0.95\textwidth]{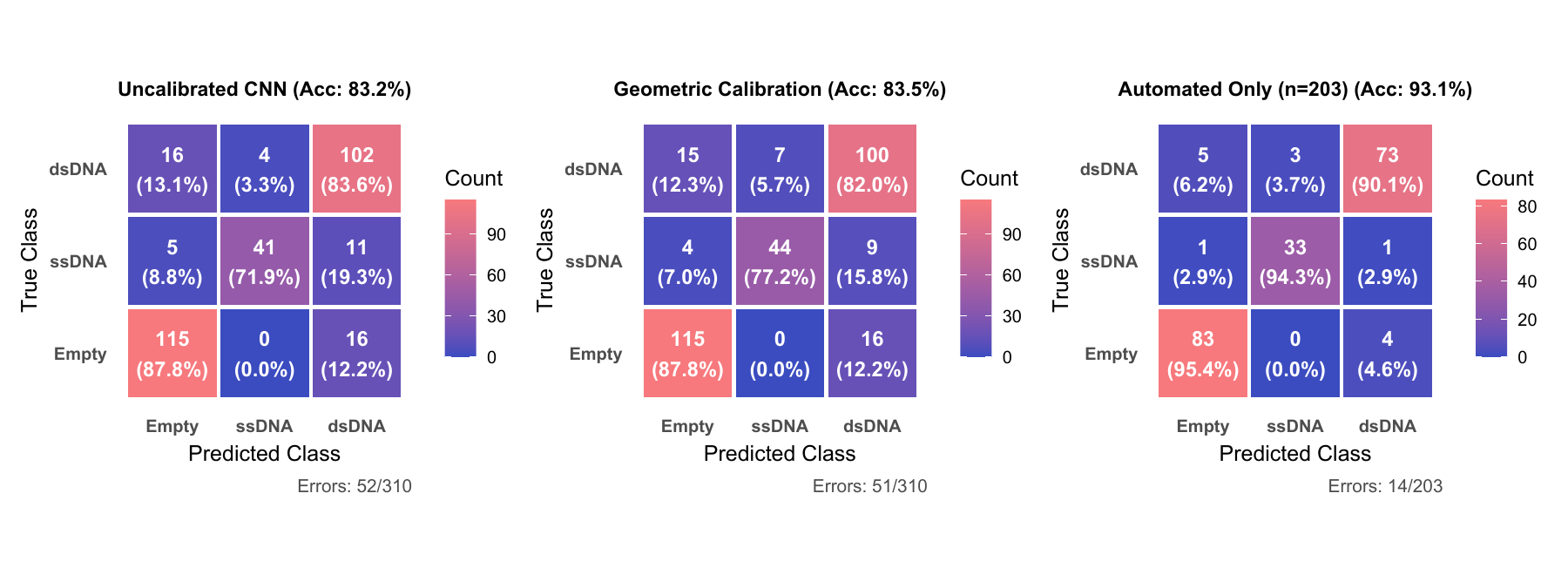}
\caption{\textbf{Confusion matrices for AAV cargo classification} ($n_{\text{val}} = 310$). \textbf{(Left)} Uncalibrated CNN: 83.2\% accuracy, 52 errors. \textbf{(Center)} After geometric calibration: 83.5\% accuracy, 51 errors. \textbf{(Right)} Automated decisions outside neutral zone ($n=203$): 93.1\% accuracy, 14 errors. Calibration alone yields marginal accuracy improvement; the substantial gain (6.9\% error rate vs.\ 16.8\% baseline) comes from uncertainty-aware deferral via neutral zones.}
\label{fig:confusion}
\end{figure}

\textbf{Key observation:} Geometric calibration improves overall accuracy only marginally (83.2\% $\to$ 83.5\%, one additional correct prediction). The ssDNA class shows targeted improvement (71.9\% $\to$ 77.2\% recall), but this modest gain understates calibration's value. The true benefit lies not in accuracy improvement, but in enabling reliable uncertainty quantification---the foundation for the neutral zone mechanism explored in Sections~\ref{sec:reliability-analysis}--\ref{sec:neutral-zone}.

\begin{table}[h]
\centering
\caption{Per-class performance metrics ($n_{\text{val}} = 310$).}
\label{tab:perclass}
\small
\begin{tabular}{lcccc}
\toprule
\textbf{Class} & \textbf{Precision} & \textbf{Recall} & \textbf{F1} & \textbf{NZ\%} \\
\midrule
Empty & .858 & .878 & .868 & 33.6 \\
ssDNA & .863 & .772 & .815 & 38.6 \\
dsDNA & .800 & .820 & .810 & 33.6 \\
\midrule
\textbf{Macro Avg} & \textbf{.840} & \textbf{.823} & \textbf{.831} & \textbf{34.5} \\
\bottomrule
\end{tabular}
\end{table}

Table~\ref{tab:perclass} summarizes per-class performance. All classes achieve F1-scores $\geq 0.81$, with precision and recall within 10 percentage points---critical balance for medical applications. The ssDNA class enters the neutral zone at highest rate (38.6\%), appropriately reflecting its lower baseline accuracy and greater inherent signal ambiguity.

\subsection{Step 1: Probability Estimate Quality}
\label{sec:calibration-quality}

Before assessing instance-level reliability (Step 2), we verify that geometric calibration improves the quality of probability estimates themselves.

\subsubsection{Calibration Metrics}

We evaluate calibration using three metrics: log loss and Brier score (strictly proper scoring rules) and Expected Calibration Error (ECE, a binned approximation of calibration quality). Table~\ref{tab:scoring-rules} presents per-class and overall results.

\begin{table}[h]
\centering
\caption{Per-class and overall calibration metrics. Geometric calibration substantially improves the hardest class (ssDNA: $-25\%$ log loss, $-37\%$ ECE) while accepting modest degradation for already well-calibrated classes. Overall proper scoring rules improve; overall ECE increases (see Remark~\ref{rem:ece-note}).}
\label{tab:scoring-rules}
\small
\begin{tabular}{llccc}
\toprule
\textbf{Class} & \textbf{Metric} & \textbf{Uncalibrated} & \textbf{Calibrated} & \textbf{Change} \\
\midrule
\multirow{3}{*}{Empty} 
  & Log Loss & 0.311 & 0.319 & $+2.7\%$ \\
  & Brier    & 0.178 & 0.183 & $+2.5\%$ \\
  & ECE      & 0.023 & 0.030 & $+28\%$ \\[4pt]
\multirow{3}{*}{ssDNA} 
  & Log Loss & 0.724 & 0.540 & $\mathbf{-25.4\%}$ \\
  & Brier    & 0.404 & 0.298 & $\mathbf{-26.3\%}$ \\
  & ECE      & 0.042 & 0.027 & $\mathbf{-37\%}$ \\[4pt]
\multirow{3}{*}{dsDNA} 
  & Log Loss & 0.407 & 0.448 & $+9.9\%$ \\
  & Brier    & 0.236 & 0.257 & $+9.0\%$ \\
  & ECE      & 0.045 & 0.037 & $-17\%$ \\[2pt]
\midrule
\multirow{3}{*}{\textbf{Overall}} 
  & Log Loss & 0.425 & 0.410 & $\mathbf{-3.4\%}$ \\
  & Brier    & 0.243 & 0.233 & $\mathbf{-3.8\%}$ \\
  & ECE      & 0.016 & 0.022 & $+33\%$ \\
\bottomrule
\end{tabular}
\end{table}

\textbf{Interpretation.} Geometric calibration performs targeted probability redistribution. The hardest class (ssDNA) improves dramatically across all metrics: log loss decreases by 25.4\%, Brier score by 26.3\%, and class-wise ECE by 37\%. This reflects the ALR transformation learning to correct the CNN's systematic under-prediction of ssDNA---the class with the highest error rate and lowest recall in the uncalibrated model.

The Empty class shows modest ECE degradation (+28\%), reflecting the trade-off inherent in learning a single affine transformation for all classes. Importantly, the strictly proper scoring rules (Log Loss, Brier) improve overall, confirming that the probability estimates are genuinely better in an information-theoretic sense.

\begin{remark}[Overall ECE Increase]
\label{rem:ece-note}
Two of three classes show per-class ECE improvement, yet overall ECE increases by 33\%. This apparent paradox arises because the two metrics use different binning: per-class ECE bins samples by $p_j$ (predicted probability for class $j$), while overall ECE bins by $\max_k p_k$ (maximum probability across classes). 

Geometric calibration performs \textit{selective sharpening}: it increases confidence for correct predictions (mean max-prob: $0.872 \to 0.885$, $p < 0.001$ by paired $t$-test) while leaving confidence for errors essentially unchanged ($0.716 \to 0.715$, $p = 0.92$). This selective sharpening improves the reliability score's discrimination ability (Section~\ref{sec:reliability-analysis}) at the cost of increased overall ECE. Specifically, the confidence bin $[0.5, 0.6]$ contributes $+0.0053$ to overall ECE (accounting for 99\% of the increase), as samples move into this transition region where accuracy (46\%) falls below confidence (55\%).

For our application, the reliability score $R$---not ECE---governs deferral decisions, making this trade-off favorable.
\end{remark}

\subsubsection{Reliability Diagrams}

Figure~\ref{fig:calibration-diagrams} visualizes per-class calibration via reliability diagrams, comparing predicted probability against empirical frequency.

\begin{figure}[h]
\centering
\includegraphics[width=0.87\textwidth]{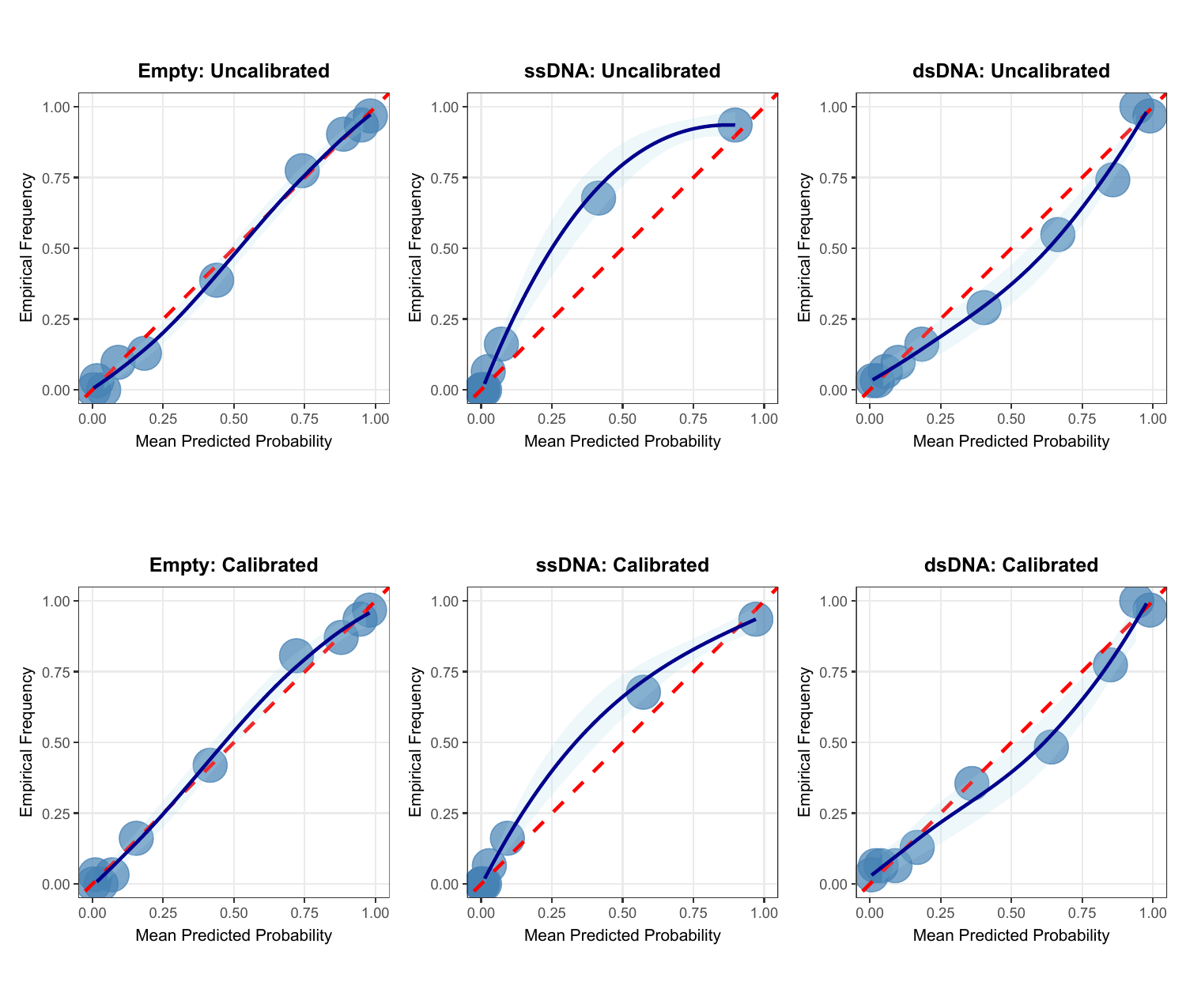}
\caption{\textbf{Per-class reliability diagrams} ($n_{\text{val}} = 310$, adaptive binning). Each panel plots mean predicted probability (x-axis) against empirical frequency (y-axis); perfect calibration follows the dashed diagonal. LOESS curves track the diagonal more closely after calibration, particularly for minority class ssDNA.}
\label{fig:calibration-diagrams}
\end{figure}

\begin{remark}[Calibration Provides Population-Level Guarantees Only]
\label{rem:population-level}
The metrics in Table~\ref{tab:scoring-rules} and Figure~\ref{fig:calibration-diagrams} confirm that Step 1 succeeds: geometric calibration produces better probability estimates. However, these are \textit{population-level} guarantees. When we say ``the model is well-calibrated,'' we mean that among all predictions with 70\% confidence, approximately 70\% are correct \textit{on average}.

This does not answer the operationally critical question: \textit{which specific predictions should we trust?} A well-calibrated model may output $\mathbf{p} = (0.7, 0.2, 0.1)$ for both an ``easy'' sample and a ``hard'' sample. Calibration alone cannot distinguish these cases. Step 2 addresses this gap.
\end{remark}

\subsection{Step 2: Instance-Level Reliability Scoring}
\label{sec:reliability-analysis}

The reliability score $R = \exp(-\lambda \cdot d_{\text{FR}}(\mathbf{p}^{\text{cal}}, \mathbf{e}_{\hat{j}}))$ translates calibrated probability vectors into scalar uncertainty measures. Unlike calibration metrics (which assess population-level properties), $R$ provides instance-level discrimination: high $R$ indicates predictions likely to be correct; low $R$ indicates predictions that warrant human review.

\subsubsection{Score Distributions}

Figure~\ref{fig:reliability-dist} presents reliability score distributions stratified by prediction outcome.

\begin{figure}[h]
\centering
\includegraphics[width=0.55\textwidth]{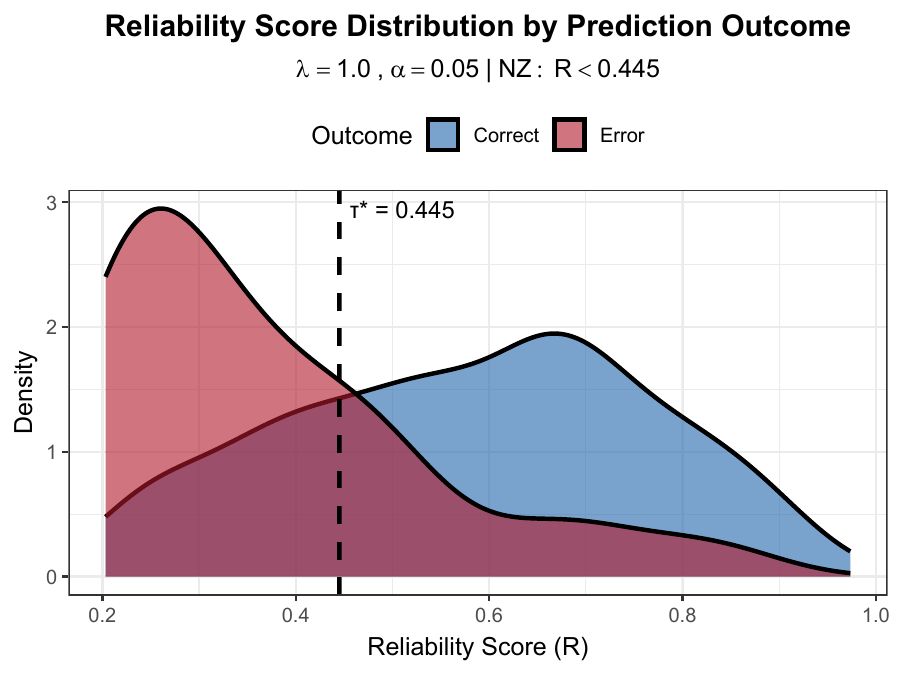}
\caption{\textbf{Reliability score distributions} ($\lambda = 1.0$). Correct predictions (blue): mean $= 0.579$. Errors (red): mean $= 0.371$. Separation $\Delta = 0.208$ (effect size $d = 1.17$). Dashed line: threshold $\tau^* = 0.445$.}
\label{fig:reliability-dist}
\end{figure}

Correct predictions ($n=259$) concentrate at high reliability with mean $R = 0.579$, while errors ($n=51$) concentrate at low reliability with mean $R = 0.371$. The separation $\Delta = 0.208$ corresponds to Cohen's $d = 1.17$---a large effect size enabling effective discrimination. The distributions exhibit bimodal separation, consistent with Theorem~\ref{thm:concentration}: reliability scores concentrate around class-conditional means with sub-Gaussian parameter $\sigma^2 \approx 0.23$.

\subsubsection{Error Detection Performance}

Figure~\ref{fig:roc-pr} evaluates the reliability score as a binary classifier for error detection. We reduce the multi-class classification problem to a binary decision task: given a predicted class $\hat{j}$, is the prediction correct ($y = \hat{j}$) or incorrect ($y \neq \hat{j}$)? The reliability score $R$ serves as the decision criterion, with low $R$ indicating likely errors.

\begin{figure}[h]
\centering
\includegraphics[width=0.97\textwidth]{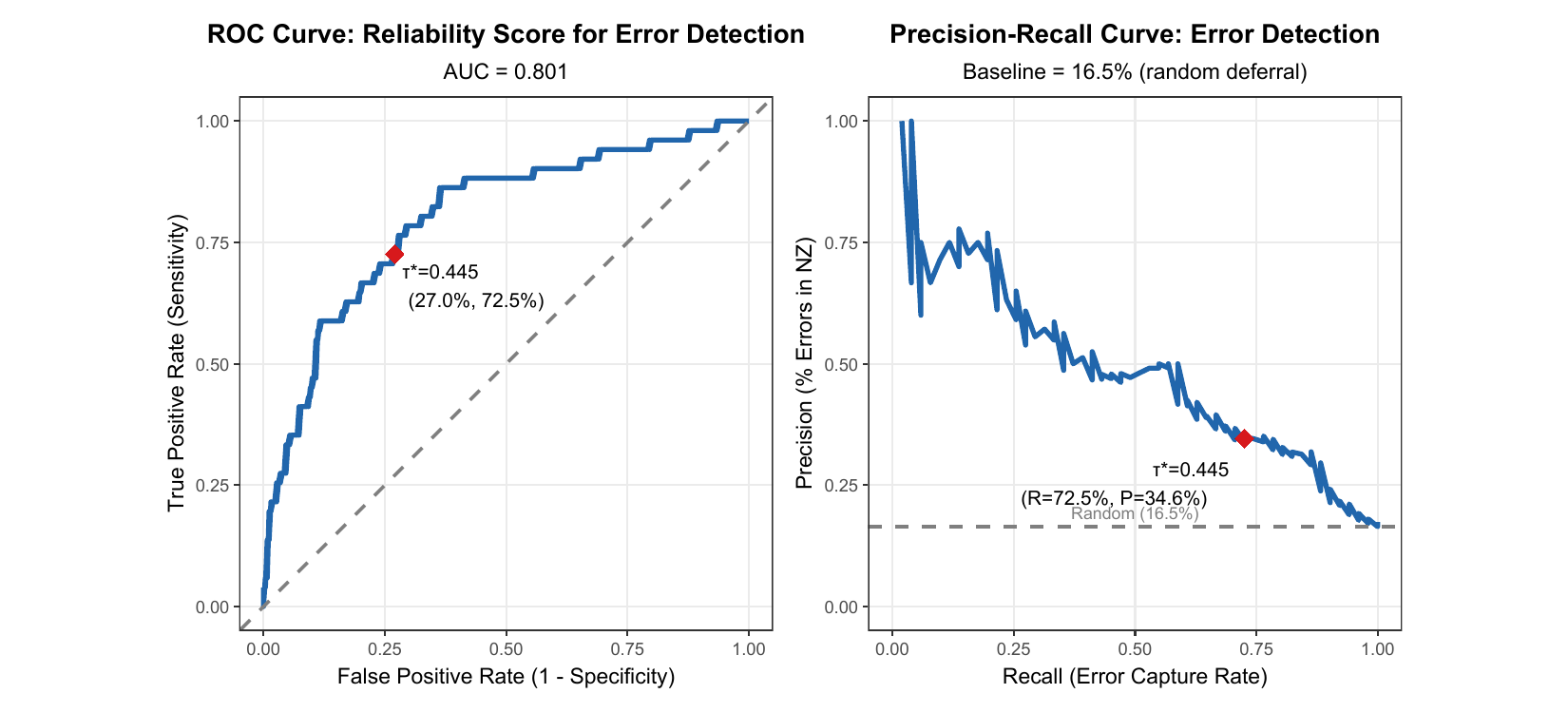}
\caption{\textbf{Error detection via reliability scores.} \textbf{(Left)} ROC curve with AUC $= 0.801$. Operating point at $\tau^* = 0.445$: 72.5\% true positive rate (error capture) at 27.0\% false positive rate. \textbf{(Right)} Precision-recall curve. At operating point: 34.6\% precision, 72.5\% recall. Horizontal line: baseline 16.5\% error rate. The $2.1\times$ improvement over random deferral demonstrates effective error identification.}
\label{fig:roc-pr}
\end{figure}

\textbf{AUC-ROC $= 0.801$} indicates strong discriminative ability for separating errors from correct predictions. At threshold $\tau^* = 0.445$: error capture (true positive rate) $= 72.5\%$; false positive rate $= 27.0\%$; precision $= 34.6\%$; improvement over random $= 72.5\% / 34.5\% = 2.1\times$. This instance-level discrimination is precisely what calibration alone (Step 1) cannot provide.

As a sanity check, we verified that simpler uncertainty heuristics---maximum predicted probability ($\max_k p_k$) and entropy ($-\sum_k p_k \log p_k$)---achieve similar AUC ($\approx 0.80$) when used to rank samples by uncertainty. This confirms that the geometric reliability score successfully captures the uncertainty information present in calibrated probabilities. The advantage of the Fisher-Rao formulation lies not in empirical superiority on a single dataset, but in the theoretical guarantees of Theorem~\ref{thm:concentration} (concentration bounds, sample efficiency) and the principled threshold selection via the $\alpha$-level constraint.

\subsection{Neutral Zone Performance: Steps 1 and 2 Combined}
\label{sec:neutral-zone}

The neutral zone mechanism combines both steps: calibrated probabilities (Step 1) feed into reliability scores (Step 2), which determine deferral decisions. Figure~\ref{fig:error-deferral} presents the complete trade-off between automation and accuracy. The Pareto frontier reveals three regimes:

\begin{figure}[h]
\centering
\includegraphics[width=0.55\textwidth]{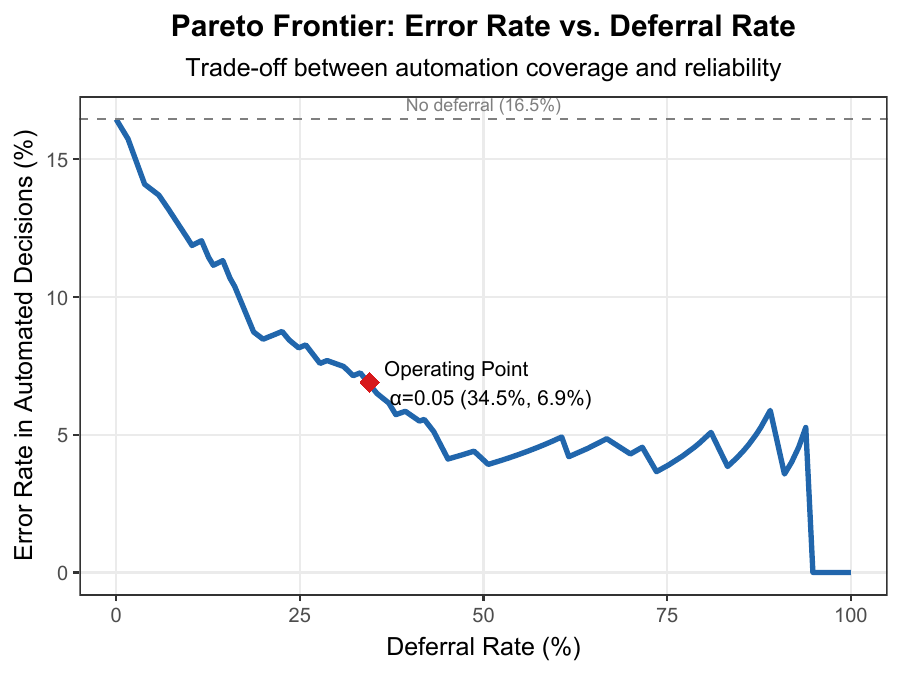}
\caption{\textbf{Pareto frontier: error rate vs. deferral rate.} Operating point (red dot): 34.5\% deferral, 6.9\% error rate. Gray horizontal line: 16.8\% baseline error rate. At 50\% deferral, error rate reaches 3.9\%.}
\label{fig:error-deferral}
\end{figure}

\textbf{High automation (0--30\% deferral):} Error rate drops rapidly from 16.8\% to $\sim$8\%---50\% reduction while maintaining 70\% automation.

\textbf{Moderate automation (30--60\%):} Continued improvement with diminishing returns. Our operating point ($\tau^* = 0.445$: 34.5\% deferral, 6.9\% error) lies at the curve's knee. At 50\% deferral, error rate reaches 3.9\%.

\textbf{Minimal automation ($>$60\%):} Eliminating errors entirely requires $\sim$80\% deferral---impractical for deployment.

\subsubsection{Quantifying the Two-Stage Framework's Value}

Table~\ref{tab:framework-value} isolates contributions from each component.

\begin{table}[h]
\centering
\caption{Decomposing the two-stage framework. Calibration alone (Step 1) provides marginal accuracy gains. Adding neutral zones (Step 2) achieves 59\% error reduction for automated decisions.}
\label{tab:framework-value}
\small
\begin{tabular}{lccc}
\toprule
\textbf{Configuration} & \textbf{Error Rate} & \textbf{Automation} & \textbf{Improvement} \\
\midrule
Uncalibrated CNN & 16.8\% & 100\% & --- \\
+ Step 1 (calibration) & 16.5\% & 100\% & $-2\%$ \\
+ Step 2 (neutral zone) & \textbf{6.9\%} & 65.5\% & $\mathbf{-59\%}$ \\
\bottomrule
\end{tabular}
\end{table}

\textbf{The key insight:} Calibration (Step 1) contributes modestly to accuracy but is essential as the foundation for Step 2. The reliability score $R$ is computed from calibrated probabilities; without proper calibration, the geometric distance to simplex vertices would not accurately reflect prediction reliability. The two steps are synergistic: neither alone achieves uncertainty-aware classification, but together they reduce automated decision errors by 59\%.

\subsubsection{Comparison to Alternative Paradigms}

\textbf{Random deferral:} At 34.5\% deferral, random selection yields no improvement (expected error rate remains 16.8\%). Our method achieves 6.9\%---a $2.4\times$ reduction.

\noindent \textbf{Conformal prediction:} Set-valued methods guarantee coverage but produce ambiguous outputs (e.g., $\{\text{Empty}, \text{dsDNA}\}$) requiring human interpretation. Our approach provides crisp predictions with scalar reliability quantification, enabling straightforward integration into clinical workflows.

\subsection{Sample Efficiency}
\label{sec:sample-efficiency}

Theorem~\ref{thm:concentration} provides concentration bounds with explicit sub-Gaussian parameter $\sigma^2 = (1 - e^{-\lambda\pi})^2/4 \approx 0.229$ for $\lambda = 1$. This has direct implications for validation set requirements. For reliability guarantee $\mathbb{P}(|R - \mathbb{E}[R]| > t) \leq \delta$ with $\delta = 0.01$ and $t = 0.1$:
\begin{equation}
n_{\text{ours}} = \frac{0.229 \times \log(200)}{0.02} \approx 61 \text{ samples}
\end{equation}

A naive Lipschitz-diameter approach, treating the reliability function as having range $\lambda\pi$ without exploiting its exponential decay structure, yields sub-Gaussian parameter $\sigma^2_{\text{naive}} = (\lambda\pi)^2/4 \approx 2.47$ for $\lambda = 1$. This requires $n_{\text{naive}} \approx 654$ samples for the same guarantees---a \textbf{10.7-fold improvement} from our tighter bound. \footnote{We emphasize that this comparison is against worst-case Lipschitz bounds that make no assumption about the reliability function's structure. Practitioners who estimate the score's range empirically from pilot data would obtain bounds similar to ours. The value of Theorem~\ref{thm:concentration} lies in providing these bounds analytically, enabling sample size determination at the study design phase before validation data is collected.}

Our empirical validation used $n = 310$ samples, providing a safety margin of $310/61 \approx 5\times$ above the theoretical minimum. This margin accommodates model misspecification and finite-sample effects not captured by asymptotic theory.

\subsection{Robustness Analysis}
\label{sec:robustness}

\subsubsection{Hyperparameter Sensitivity}

\textbf{Regularization ($\lambda_1, \lambda_2$):} Varying across $\{0.001, 0.01, 0.1\}$ yields error rates in $[6.7\%, 7.0\%]$---stable across two orders of magnitude.

\noindent \textbf{Reliability parameter ($\lambda$):} Testing $\lambda \in \{0.5, 1.0, 1.5, 2.0\}$ produces identical AUC-ROC $= 0.80$ across all values, indicating robustness to this parameter choice.

\noindent \textbf{Threshold ($\tau^*$):} Perturbations of $\pm 0.05$ yield error rates of 8.4\% (at $\tau^* - 0.05$) and 5.5\% (at $\tau^* + 0.05$), with corresponding deferral rates of 26.8\% and 41.3\%. The smooth Pareto frontier indicates stable operating regions.

\subsubsection{Cross-Validation}

Five-fold cross-validation yields: error rate $7.1\% \pm 5.5\%$; error capture $74.6\% \pm 13.4\%$; deferral rate $35.8\% \pm 7.0\%$. The relatively high variance reflects the small validation set size ($n = 310$) and class imbalance (ssDNA comprises only 18.4\% of samples).

\subsubsection{Method Comparison}

Table~\ref{tab:method-comparison} compares geometric calibration against established post-hoc calibration methods: temperature scaling \cite{guo2017calibration}, Platt scaling \cite{platt1999}, and isotonic regression \cite{zadrozny2002}. For fair comparison, all methods use the same reliability score formula $R = \exp(-\lambda \cdot d_{FR}(\mathbf{p}^{\text{cal}}, \mathbf{e}_{\hat{j}}))$ applied to their respective calibrated probabilities, with thresholds learned to achieve 34.5\% deferral.

\begin{table}[h]
\centering
\caption{Post-hoc calibration method comparison at 34.5\% deferral. All methods achieve similar error detection performance (AUC 0.78-0.80, capture 70-73\%) when combined with reliability-score-based deferral. Geometric calibration achieves highest AUC but the advantage is modest. Platt scaling achieves best overall accuracy (84.5\%) but slightly lower error detection. These results suggest the neutral zone mechanism itself—rather than the specific calibration method—drives operational gains.}
\label{tab:method-comparison}
\small
\begin{tabular}{lccccc}
\toprule
\textbf{Method} & \textbf{Errors} & \textbf{Accuracy} & \textbf{ECE} $\downarrow$ & \textbf{AUC} $\uparrow$ & \textbf{Capture} \\
\midrule
Uncalibrated & 52 & 83.2\% & 0.016 & 0.797 & 71.2\% \\
Temperature & 52 & 83.2\% & 0.022 & 0.797 & 73.1\% \\
Platt (OvR) & 48 & \textbf{84.5\%} & 0.021 & 0.778 & 70.8\% \\
Isotonic & 50 & 83.9\% & 0.034 & 0.784 & 70.0\% \\
\textbf{Geometric} & 51 & 83.5\% & 0.022 & \textbf{0.801} & \textbf{72.5\%} \\
\bottomrule
\end{tabular}
\end{table}

All methods achieve similar performance, with differences of only 2--3 percentage points across metrics. Platt scaling achieves the best overall accuracy (84.5\%, 48 errors), while geometric calibration achieves the highest error detection AUC (0.801). These modest differences lead to an important conclusion about the paper's contribution.

\textbf{Interpretation.} The operational gains demonstrated in Section~\ref{sec:neutral-zone}---72.5\% error capture at 34.5\% deferral, yielding 6.9\% automated error rate---arise primarily from the neutral zone mechanism itself, which can be applied to \textit{any} calibration method. The specific contribution of geometric calibration is therefore not empirical superiority but rather:
\begin{enumerate}[leftmargin=*,itemsep=2pt]
\item A principled multi-class generalization of Platt scaling grounded in information geometry (Proposition~\ref{prop:platt}), filling a theoretical gap in the calibration literature;
\item Formal consistency guarantees with explicit convergence rate (Theorem~\ref{thm:consistency});
\item Concentration bounds with explicit sub-Gaussian parameters enabling sample size calculations before data collection (Theorem~\ref{thm:concentration}).
\end{enumerate}

For applications where empirical performance is the sole criterion, simpler methods like temperature scaling may suffice. However, for settings requiring formal validation documentation---regulatory submissions, safety-critical deployments, or expensive validation studies where sample size must be justified \textit{a priori}---the theoretical foundations of geometric calibration provide value that empirical benchmarks alone cannot capture.

\section{Conclusion}
\label{sec:conclusion}

We have developed a geometric framework for uncertainty-aware classification that addresses both calibration and instance-level uncertainty quantification. The key insight is that calibration alone---even perfect calibration---provides only population-level guarantees, leaving unanswered the operationally critical question: \textit{which specific predictions should we trust?} Our two-stage framework answers this by combining geometric calibration with reliability scoring that translates calibrated probabilities into actionable uncertainty measures.

\paragraph{Main contributions.}
The primary operational contribution is the reliability scoring mechanism, which enables principled deferral of uncertain predictions. On AAV classification, this mechanism captures 72.5\% of errors while deferring 34.5\% of samples---performance that is largely independent of the specific calibration method used (Table~\ref{tab:method-comparison}). The geometric framework provides the theoretical foundations: Proposition~\ref{prop:platt} establishes geometric calibration as the natural multi-class generalization of Platt scaling, Theorem~\ref{thm:consistency} provides consistency guarantees, and Theorem~\ref{thm:concentration} yields concentration bounds with explicit sub-Gaussian parameters enabling sample size calculations before data collection.

\paragraph{Empirical findings.}
Validation on AAV classification demonstrates that the two-stage framework achieves substantial error reduction: automated decision errors drop from 16.8\% to 6.9\%. Crucially, calibration alone improves accuracy only marginally (83.2\% $\to$ 83.5\%); the operational benefit arises from reliability-based deferral. This separation clarifies that while geometric calibration provides principled theoretical foundations, the neutral zone mechanism itself---applicable to any well-calibrated output---drives the practical gains.

\paragraph{Limitations.}
Several limitations warrant acknowledgment. First, empirical validation is confined to a single application domain with $c = 3$ classes; the framework's behavior in high-dimensional settings ($c \gg 3$) and under distribution shift remains to be demonstrated. Second, the evidence for faster-than-$n^{-1/2}$ convergence (Remark~\ref{rem:finite-sample}) is suggestive but inconclusive, and Conjecture~\ref{conj:np} on Neyman-Pearson optimality lacks rigorous proof. Third, the cross-validation variance (Section~\ref{sec:robustness}) indicates that operating point performance exhibits meaningful uncertainty at $n = 310$.

\paragraph{Future directions.}
Several open problems merit investigation: (i) rigorous proof of Conjecture~\ref{conj:np} on Neyman-Pearson optimality; (ii) extension to high-dimensional settings ($c \gg 3$) and analysis under distribution shift; and (iii) integration with the r-power methodology \cite{jasa2025} for distribution-free neutral zones in multi-class settings. The connection between Fisher-Rao geometry and optimal decision theory suggested by Remark~\ref{rem:bhattacharyya} merits deeper exploration.

\paragraph{Broader context.}
As AI systems increasingly inform high-stakes decisions in medicine, criminal justice, and autonomous systems, the gap between calibration and actionable uncertainty quantification becomes critical. Well-calibrated probabilities tell us that 70\% of predictions at confidence 0.7 are correct; they do not tell us which 30\% will fail. This work contributes formal foundations for bridging that gap, offering instance-level reliability scores with theoretical guarantees that may prove particularly valuable in applications requiring rigorous validation documentation.



\bibliographystyle{unsrt}
\bibliography{shortened}

\appendix

\newpage 

\section{Proofs}
\label{app:proofs}

\subsection{Proof of Theorem~\ref{thm:consistency}}
\label{app:proof-consistency}

We prove consistency $\hat{T}_n \xrightarrow{p} T^*$ and the rate $\|\hat{T}_n - T^*\| = O_p(n^{-1/2})$ using the M-estimation framework of van der Vaart \cite{vandervaart1998}.

\subsubsection*{Preliminary: M-Estimation Setup}

Our estimator minimizes an empirical loss:
\begin{equation}
\hat{T}_n = \argmin_{T \in \mathcal{T}} \mathcal{L}_n(T), \quad \mathcal{L}_n(T) = \frac{1}{n}\sum_{i=1}^n \ell(T; \mathbf{p}_i, y_i) + \lambda_1\|A - I\|_F^2 + \lambda_2\|\mathbf{b}\|^2
\end{equation}
where the per-sample loss is $\ell(T; \mathbf{p}, \mathbf{y}) = -\sum_{j=1}^c y_j \log p_j^{\text{cal}}(T)$.

The population version is:
\begin{equation}
T^* = \argmin_{T \in \mathcal{T}} \mathcal{L}^*(T), \quad \mathcal{L}^*(T) = \mathbb{E}[\ell(T; \mathbf{p}, \mathbf{y})] + \lambda_1\|A - I\|_F^2 + \lambda_2\|\mathbf{b}\|^2
\end{equation}

\begin{remark}[Van der Vaart's Convention]
Van der Vaart's theorems are stated for \textit{maximizing} a criterion $M_n(\theta)$. To apply them to our minimization problem, we work with $M_n(T) = -\mathcal{L}_n(T)$ and $M^*(T) = -\mathcal{L}^*(T)$. All conditions below are verified for this flipped objective.
\end{remark}

\subsubsection*{Part I: Consistency}

\textbf{Van der Vaart, Theorem 5.7:} Let $\hat{\theta}_n$ maximize $M_n(\theta)$ over a set $\Theta$. If:
\begin{enumerate}[label=(\roman*)]
\item $\Theta$ is compact,
\item $\theta \mapsto M(\theta)$ is continuous,
\item $\sup_{\theta \in \Theta} |M_n(\theta) - M(\theta)| \xrightarrow{p} 0$,
\item For all $\varepsilon > 0$: $\sup_{\theta: d(\theta,\theta_0) \geq \varepsilon} M(\theta) < M(\theta_0)$ (well-separation),
\end{enumerate} then $\hat{\theta}_n \xrightarrow{p} \theta_0$.

We verify each condition for $M_n = -\mathcal{L}_n$ and $M^* = -\mathcal{L}^*$.

\paragraph{Condition (i): Compactness of $\mathcal{T}$.}

Define the parameter space:
\begin{equation}
\mathcal{T} = \{(A, \mathbf{b}) : \|A\|_F \leq M_A, \|\mathbf{b}\| \leq M_b, A \succeq \delta I, \tr(A) = c-1\}
\end{equation}
where $M_A, M_b, \delta > 0$ are determined by the regularization strength. The lower bound $\delta > 0$ on eigenvalues ensures $\mathcal{T}$ is closed (and hence compact, being bounded). In practice, the regularization term $\lambda_1\|A - I\|_F^2$ keeps $A$ well-conditioned: for $\|A - I\|_F \leq M_A$ with $M_A < 1$, we have $\lambda_{\min}(A) \geq 1 - M_A > 0$.

\begin{proof}[Proof of compactness]
We show $\mathcal{T}$ is closed and bounded in $\mathbb{R}^{(c-1)^2 + (c-1)}$.

\noindent \textbf{Closed:} $\mathcal{T}$ is the intersection of:
\begin{itemize}
\item $\{(A,\mathbf{b}) : \|A\|_F \leq M_A\}$ is closed (continuous norm, inverse image of $[0, M_A]$)
\item $\{(A,\mathbf{b}) : \|\mathbf{b}\| \leq M_b\}$ is closed
\item $\{(A,\mathbf{b}) : A \succeq \delta I\}$ is closed (smallest eigenvalue is continuous, inverse image of $[\delta, \infty)$)
\item $\{(A,\mathbf{b}) : \tr(A) = c-1\}$ is closed (continuous trace, inverse image of singleton)
\end{itemize}
Finite intersection of closed sets is closed.

\noindent \textbf{Bounded:} $\|A\|_F \leq M_A$ and $\|\mathbf{b}\| \leq M_b$ directly.

By the Heine-Borel theorem, $\mathcal{T}$ is compact.
\end{proof}

\paragraph{Condition (ii): Continuity of $\mathcal{L}^*$.}

The regularization terms $\lambda_1\|A - I\|_F^2 + \lambda_2\|\mathbf{b}\|^2$ are continuous. For the expectation term, we show $T \mapsto \mathbb{E}[\ell(T; \mathbf{p}, \mathbf{y})]$ is continuous.

\begin{proof}[Proof of continuity]
\textit{Step 1: Pointwise continuity.} For each fixed $(\mathbf{p}, \mathbf{y})$, the map $T \mapsto \ell(T; \mathbf{p}, \mathbf{y})$ is continuous because:
\begin{equation}
\ell(T; \mathbf{p}, \mathbf{y}) = -\sum_j y_j \log p_j^{\text{cal}}(T), \quad p_j^{\text{cal}}(T) = \text{softmax}_j(A \cdot \text{ALR}(\mathbf{p}) + \mathbf{b})
\end{equation}

This is a composition of:
\begin{itemize}
\item $(A, \mathbf{b}) \mapsto A\mathbf{z} + \mathbf{b}$ (linear, continuous)
\item $\mathbf{z}^{\text{cal}} \mapsto \text{softmax}(\mathbf{z}^{\text{cal}})$ (smooth, $C^\infty$)
\item $\mathbf{p}^{\text{cal}} \mapsto -\sum_j y_j \log p_j^{\text{cal}}$ (continuous when $p_j > 0$)
\end{itemize}

Since $p_j^{\text{cal}} \geq \epsilon' > 0$ for all $T \in \mathcal{T}$ (Lemma~\ref{lem:bounded-loss}), the logarithm is well-defined and continuous.

\textit{Step 2: Dominated convergence.} By Lemma~\ref{lem:bounded-loss}, there exists a constant $M < \infty$ such that $|\ell(T; \mathbf{p}, \mathbf{y})| \leq M$ for all $T \in \mathcal{T}$, with $\mathbb{E}[M] = M < \infty$.

For any sequence $T_n \to T_0$ in $\mathcal{T}$:
\begin{equation}
\lim_{n \to \infty} \mathbb{E}[\ell(T_n; \mathbf{p}, \mathbf{y})] = \mathbb{E}\left[\lim_{n \to \infty} \ell(T_n; \mathbf{p}, \mathbf{y})\right] = \mathbb{E}[\ell(T_0; \mathbf{p}, \mathbf{y})]
\end{equation}

Therefore $\mathcal{L}^*$ is continuous.
\end{proof}

\paragraph{Condition (iii): Uniform Convergence (Glivenko-Cantelli).}

We must show $\sup_{T \in \mathcal{T}} |\mathcal{L}_n(T) - \mathcal{L}^*(T)| \xrightarrow{p} 0$.

Since the regularization terms are deterministic, it suffices to prove:
\begin{equation}
\sup_{T \in \mathcal{T}} \left|\frac{1}{n}\sum_{i=1}^n \ell(T; \mathbf{p}_i, y_i) - \mathbb{E}[\ell(T; \mathbf{p}, \mathbf{y})]\right| \xrightarrow{p} 0
\end{equation}

\noindent \textbf{Van der Vaart, Example 19.8 (Pointwise Compact Class):} Let $\mathcal{F} = \{f_\theta : \theta \in \Theta\}$ be a parametric class indexed by a compact metric space $\Theta$ such that $\theta \mapsto f_\theta(x)$ is continuous for every $x$, with integrable envelope function $F$. Then the $L_1$-bracketing numbers of $\mathcal{F}$ are finite and hence $\mathcal{F}$ is Glivenko-Cantelli.

\begin{proof}[Verification of Glivenko-Cantelli conditions]
We start with verifying the conditions used in example 19.8 of van der Vaart.
\begin{enumerate}
\item $\mathcal{T}$ is a compact metric space (Condition (i)) $\clubsuit$
\item For each fixed $(\mathbf{p}, \mathbf{y})$, the map $T \mapsto \ell(T; \mathbf{p}, \mathbf{y})$ is continuous (Condition (ii)) $\clubsuit$
\item Integrable envelope: $|\ell(T; \mathbf{p}, \mathbf{y})| \leq M$ for all $T \in \mathcal{T}$ with $\mathbb{E}[M] = M < \infty$ (Lemma~\ref{lem:bounded-loss}) $\clubsuit$
\end{enumerate}

By Example 19.8, the function class is Glivenko-Cantelli.
\end{proof}

\begin{lemma}[Bounded Loss]
\label{lem:bounded-loss}
Under the assumption $p_j^{\text{CNN}} \geq \epsilon > 0$ for all $j = 1, \ldots, c$, there exists a constant $M < \infty$ (depending only on $\epsilon, M_A, M_b, c$) such that:
\begin{equation}
\sup_{T \in \mathcal{T}} |\ell(T; \mathbf{p}, \mathbf{y})| \leq M
\end{equation}
\end{lemma}

\begin{proof}
We establish the bound through a chain of four inequalities.

\textit{Step (a): Bounded ALR inputs.} Since $p_j^{\text{CNN}} \geq \epsilon$ and $\sum_j p_j = 1$:
\begin{equation}
p_j^{\text{CNN}} \leq 1 - \sum_{k \neq j} p_k^{\text{CNN}} \leq 1 - (c-1)\epsilon
\end{equation}

The ALR coordinates satisfy:
\begin{equation}
|z_k| = \left|\log\frac{p_k}{p_c}\right| \leq \max\left\{\log\frac{1-(c-1)\epsilon}{\epsilon}, \log\frac{\epsilon}{1-(c-1)\epsilon}\right\} =: B_z < \infty
\end{equation}

Therefore $\|\mathbf{z}\|_2 \leq B_z\sqrt{c-1}$.

\textit{Step (b): Bounded calibrated logits.} For any $T = (A, \mathbf{b}) \in \mathcal{T}$:
\begin{equation}
\|\mathbf{z}^{\text{cal}}\|_2 = \|A\mathbf{z} + \mathbf{b}\|_2 \leq \|A\|_{op}\|\mathbf{z}\|_2 + \|\mathbf{b}\|_2 \leq M_A \cdot B_z\sqrt{c-1} + M_b =: B_{\text{cal}}
\end{equation}

where we used $\|A\|_{op} \leq \|A\|_F \leq M_A$. Thus $\|\mathbf{z}^{\text{cal}}\|_\infty \leq B_{\text{cal}}$.

\textit{Step (c): Lower bound on calibrated probabilities.} For softmax with bounded inputs:
\begin{equation}
p_j^{\text{cal}} = \frac{e^{z_j^{\text{cal}}}}{\sum_{k=1}^c e^{z_k^{\text{cal}}}} \geq \frac{e^{-B_{\text{cal}}}}{c \cdot e^{B_{\text{cal}}}} = \frac{1}{c}e^{-2B_{\text{cal}}} =: \epsilon' > 0
\end{equation}

\textit{Step (d): Bounded log-loss.} Since $\mathbf{y}$ is one-hot encoded ($\sum_j y_j = 1$, exactly one $y_j = 1$):
\begin{equation}
|\ell(T; \mathbf{p}, \mathbf{y})| = \left|-\sum_j y_j \log p_j^{\text{cal}}\right| = -\log p_{\hat{y}}^{\text{cal}} \leq -\log(\epsilon') = \log(c) + 2B_{\text{cal}} =: M
\end{equation} where $\hat{y}$ is the true class. Critically, $M$ depends only on constants $(\epsilon, M_A, M_b, c)$, not on the random data.
\end{proof}

\paragraph{Condition (iv): Well-Separation (Identifiability).}

For the flipped objective $M^* = -\mathcal{L}^*$, van der Vaart's condition requires:
\begin{equation}
\sup_{T: \|T - T^*\| \geq \varepsilon} M^*(T) < M^*(T^*)
\end{equation}

Equivalently (multiplying by $-1$ and reversing inequality):
\begin{equation}
\inf_{T: \|T - T^*\| \geq \varepsilon} \mathcal{L}^*(T) > \mathcal{L}^*(T^*)
\end{equation}

This follows from strong convexity.

\begin{claim}
$\mathcal{L}^*$ is strongly convex with modulus $\mu = \min\{2\lambda_1, 2\lambda_2\} > 0$.
\end{claim}

\begin{proof}
The Hessian decomposes as:
\begin{equation}
\nabla^2 \mathcal{L}^*(T) = \mathbb{E}[\nabla^2_T \ell(T; \mathbf{p}, \mathbf{y})] + 2\lambda_1 I_{(c-1)^2} + 2\lambda_2 I_{c-1}
\end{equation}

To connect Lemma~\ref{lem:softmax-hessian} to the parameter Hessian, we use the chain rule. Let $J_T = \partial \mathbf{z}^{\text{cal}} / \partial T$ denote the Jacobian of the map $T \mapsto A \cdot \text{ALR}(\mathbf{p}) + \mathbf{b}$. The Hessian of the per-sample loss decomposes as:
\begin{equation}
\nabla^2_T \ell(T; \mathbf{p}, y) = J_T^\top \Sigma_{\mathbf{p}^{\text{cal}}} J_T
\end{equation}
where $\Sigma_{\mathbf{p}} = \text{diag}(\mathbf{p}) - \mathbf{p}\mathbf{p}^\top$ is the softmax Hessian. By Lemma~\ref{lem:softmax-hessian}, $\Sigma_{\mathbf{p}^{\text{cal}}} \succeq \epsilon^2 I_{c-1}$, which implies:
\begin{equation}
\mathbb{E}[\nabla^2_T \ell] = \mathbb{E}[J_T^\top \Sigma_{\mathbf{p}^{\text{cal}}} J_T] \succeq 0
\end{equation}

The data term is positive semi-definite but not uniformly positive definite, since $J_T \in \mathbb{R}^{(c-1) \times [(c-1)^2 + (c-1)]}$ has rank at most $c-1$, leaving a non-trivial nullspace. The regularization terms provide the requisite uniform curvature:
\begin{equation}
\nabla^2 \mathcal{L}^*(T) \succeq \mu I, \quad \mu = \min\{2\lambda_1, 2\lambda_2\} > 0
\end{equation}

Strong convexity implies: for any $T \neq T^*$,
\begin{equation}
\mathcal{L}^*(T) \geq \mathcal{L}^*(T^*) + \nabla\mathcal{L}^*(T^*)^\top(T - T^*) + \frac{\mu}{2}\|T - T^*\|^2
\end{equation}

At the minimizer, $\nabla\mathcal{L}^*(T^*) = 0$, so:
\begin{equation}
\mathcal{L}^*(T) \geq \mathcal{L}^*(T^*) + \frac{\mu}{2}\|T - T^*\|^2
\end{equation}

For all $T$ with $\|T - T^*\| \geq \varepsilon$:
\begin{equation}
\inf_{\|T - T^*\| \geq \varepsilon} \mathcal{L}^*(T) \geq \mathcal{L}^*(T^*) + \frac{\mu\varepsilon^2}{2} > \mathcal{L}^*(T^*)
\end{equation}
\end{proof}

Strong convexity also ensures $T^*$ is the unique minimizer. Existence follows from compactness of $\mathcal{T}$ and continuity of $\mathcal{L}^*$ (Weierstrass extreme value theorem).

\begin{conclusion}
All four conditions of van der Vaart Theorem 5.7 are satisfied. Therefore:
\begin{equation}
\boxed{\hat{T}_n \xrightarrow{p} T^*}
\end{equation}
\end{conclusion}

\subsubsection*{Part II: Rate}

\textbf{Van der Vaart, Theorem 5.23:} Let $\hat{\theta}_n$ maximize $M_n(\theta)$. If in addition to Theorem 5.7:
\begin{enumerate}[label=(\roman*)]
\item $\theta \mapsto M_n(\theta)$ is twice continuously differentiable,
\item $\theta_0 \in \text{interior}(\Theta)$,
\item The map $\theta \mapsto PM_\theta$ admits a second-order Taylor expansion at $\theta_0$ with non-singular Hessian $V_{\theta_0}$,
\item $\mathbb{P}_n M_{\hat{\theta}_n} \geq \sup_\theta \mathbb{P}_n M_\theta - o_P(n^{-1})$,
\item Stochastic equicontinuity holds,
\end{enumerate}
then:
\begin{equation}
\sqrt{n}(\hat{\theta}_n - \theta_0) \xrightarrow{d} N(0, V_{\theta_0}^{-1} \text{Var}(\dot{M}_{\theta_0}) V_{\theta_0}^{-1})
\end{equation}

\paragraph{Condition (i): Twice Continuous Differentiability.}

The loss $\ell(T; \mathbf{p}, \mathbf{y}) = -\sum_j y_j \log p_j^{\text{cal}}(T)$ is $C^\infty$ in $T$ since softmax is infinitely differentiable and $p_j^{\text{cal}} \geq \epsilon' > 0$ ensures logarithm is smooth. Therefore $\mathcal{L}_n$ and $\mathcal{L}^*$ are twice continuously differentiable. $\clubsuit$

\paragraph{Condition (ii): Interior Solution.}

For sufficiently large $M_A$ and $M_b$ (chosen based on the data-generating process), the minimizer $T^*$ satisfies $\|A^*\|_F < M_A$ and $\|\mathbf{b}^*\| < M_b$ strictly. The regularization terms penalize large parameters, ensuring the solution lies strictly in the interior. $\clubsuit$

\paragraph{Condition (iii): Non-Singular Hessian.}

We have $H := \nabla^2 \mathcal{L}^*(T^*) \succeq \mu I$ with $\mu > 0$ from strong convexity (Part I). $\clubsuit$

\paragraph{Condition (iv): Optimality.}

Our $\hat{T}_n$ is the exact minimizer: $\mathcal{L}_n(\hat{T}_n) = \min_T \mathcal{L}_n(T)$. Flipping signs: $M_n(\hat{T}_n) = \max_T M_n(T)$, which satisfies the condition with $o_P(n^{-1}) = 0$. $\clubsuit$

\paragraph{Condition (v): Stochastic Equicontinuity.}

The gradient satisfies $\nabla_T \ell(T; \mathbf{p}, \mathbf{y}) = J_T^\top(\mathbf{p}^{\text{cal}} - \mathbf{y})$ where $J_T$ is the Jacobian. On compact $\mathcal{T}$, both $J_T$ and $\mathbf{p}^{\text{cal}} - \mathbf{y}$ are bounded, giving $\|\nabla_T \ell\|_2 \leq C < \infty$. This implies finite variance: $\text{Var}(\nabla_T \ell(T^*; \mathbf{p}, \mathbf{y})) \leq C^2 < \infty$. $\clubsuit$

\begin{lemma}[Softmax Hessian Eigenvalue]
\label{lem:softmax-hessian}
For $\mathbf{p} \in \Delta^{c-1}$ with $p_j \geq \epsilon$ for all $j$, the softmax Hessian satisfies:
\begin{equation}
\lambda_{\min}(\Sigma_{\mathbf{p}}) \geq \epsilon^2, \quad \Sigma_{\mathbf{p}} = \text{diag}(\mathbf{p}) - \mathbf{p}\mathbf{p}^\top
\end{equation}
\end{lemma}

\begin{proof}
Work with the $(c-1) \times (c-1)$ matrix $H = \text{diag}(p_1, \ldots, p_{c-1}) - \mathbf{p}_{1:c-1}\mathbf{p}_{1:c-1}^\top$. For any unit vector $\mathbf{v} \in \mathbb{R}^{c-1}$ with $\|\mathbf{v}\|_2 = 1$:
\begin{align}
\mathbf{v}^\top H \mathbf{v} &= \sum_{i=1}^{c-1} p_i v_i^2 - \left(\sum_{i=1}^{c-1} p_i v_i\right)^2 \\
&\geq \sum_{i=1}^{c-1} p_i v_i^2 - (1-p_c)\sum_{i=1}^{c-1} p_i v_i^2 \quad \text{(Cauchy-Schwarz: $(\sum a_i b_i)^2 \leq (\sum a_i)(\sum a_i b_i^2)$)} \\
&= p_c \sum_{i=1}^{c-1} p_i v_i^2 \\
&\geq p_c \cdot \min_{1 \leq i \leq c-1} p_i \cdot \underbrace{\sum_{i=1}^{c-1} v_i^2}_{= \|\mathbf{v}\|_2^2 = 1} \\
&\geq \epsilon \cdot \epsilon = \epsilon^2
\end{align}
\end{proof}

\begin{conclusion}
By van der Vaart Theorem 5.23:
\begin{equation}
\sqrt{n}(\hat{T}_n - T^*) \xrightarrow{d} N(0, H^{-1}VH^{-1})
\end{equation}

where $H = \nabla^2\mathcal{L}^*(T^*)$ and $V = \text{Var}(\nabla\ell(T^*; \mathbf{p}, \mathbf{y}))$.

This implies:
\begin{equation}
\boxed{\|\hat{T}_n - T^*\| = O_p(n^{-1/2})}
\end{equation}
\end{conclusion}

\subsection{Proof of Theorem~\ref{thm:concentration}}
\label{app:proof-concentration}

We prove concentration of the reliability score $R(\mathbf{p}) = \exp(-\lambda \cdot \dFR(\mathbf{p}, \mathbf{e}_{\hat{j}}))$ where $\hat{j} = \argmax_k p_k$.

\begin{remark}[Proof Strategy]
The reliability function involves an $\argmax$ operation, which could introduce discontinuities at decision boundaries where $p_i = p_j$ for some $i \neq j$. We first prove $R$ is continuous everywhere (Steps 1--3), then use this continuity with the bounded range property to apply Hoeffding's inequality (Steps 4--6).
\end{remark}

\subsubsection*{Step 1: Decision Region Decomposition}

The simplex $\Delta^{c-1}$ partitions into $c$ decision regions:
\begin{equation}
\Delta_j = \{\mathbf{p} \in \Delta^{c-1} : p_j = \max_k p_k\}, \quad j = 1, \ldots, c
\end{equation}

Within each $\Delta_j$, the predicted class is fixed at $\hat{j} = j$, so the map $R_j(\mathbf{p}) = \exp(-\lambda \cdot \dFR(\mathbf{p}, \mathbf{e}_j))$ is smooth (composition of smooth functions: Fisher-Rao distance and exponential).

The decision boundaries $\mathcal{B} = \{\mathbf{p} : p_i = p_j \text{ for some } i \neq j\}$ form a set of measure zero in $\Delta^{c-1}$ (union of lower-dimensional submanifolds).

\subsubsection*{Step 2: Lipschitz Property Within Regions}

We establish Lipschitz continuity within each decision region, which will be used to prove global continuity in Step 3.

\begin{claim}
$R_j$ is $\lambda$-Lipschitz on $\Delta_j$.
\end{claim}

\begin{proof}
The Fisher-Rao distance satisfies the triangle inequality. For any $\mathbf{p}, \mathbf{q} \in \Delta^{c-1}$:
\begin{equation}
|\dFR(\mathbf{p}, \mathbf{e}_j) - \dFR(\mathbf{q}, \mathbf{e}_j)| \leq \dFR(\mathbf{p}, \mathbf{q})
\end{equation}

For the exponential function $f(d) = e^{-\lambda d}$, we have $|f'(d)| = \lambda e^{-\lambda d} \leq \lambda$ for all $d \geq 0$. By the mean value theorem and chain rule, for $\mathbf{p}, \mathbf{q} \in \Delta_j$:
\begin{equation}
|R_j(\mathbf{p}) - R_j(\mathbf{q})| = |f(\dFR(\mathbf{p}, \mathbf{e}_j)) - f(\dFR(\mathbf{q}, \mathbf{e}_j))| \leq \lambda \cdot |\dFR(\mathbf{p}, \mathbf{e}_j) - \dFR(\mathbf{q}, \mathbf{e}_j)|
\end{equation}

Applying the triangle inequality:
\begin{equation}
|R_j(\mathbf{p}) - R_j(\mathbf{q})| \leq \lambda \cdot \dFR(\mathbf{p}, \mathbf{q})
\end{equation}
\end{proof}

\subsubsection*{Step 3: Continuity Across Boundaries}

\begin{remark}[Key Insight]
Although the $\argmax$ operation switches between regions at boundaries, the function values match exactly at these boundaries.
\end{remark}

Consider a boundary point between $\Delta_i$ and $\Delta_j$ where $p_i = p_j$ (both are maximal). The Fisher-Rao distance to a vertex $\mathbf{e}_k$ is:
\begin{equation}
\dFR(\mathbf{p}, \mathbf{e}_k) = 2\arccos\left(\sum_{j=1}^c \sqrt{p_j \cdot \delta_{jk}}\right) = 2\arccos(\sqrt{p_k})
\end{equation}

When $p_i = p_j$ at a boundary:
\begin{equation}
\dFR(\mathbf{p}, \mathbf{e}_i) = 2\arccos(\sqrt{p_i}) = 2\arccos(\sqrt{p_j}) = \dFR(\mathbf{p}, \mathbf{e}_j)
\end{equation}

Therefore:
\begin{equation}
R_i(\mathbf{p}) = \exp(-\lambda \cdot \dFR(\mathbf{p}, \mathbf{e}_i)) = \exp(-\lambda \cdot \dFR(\mathbf{p}, \mathbf{e}_j)) = R_j(\mathbf{p})
\end{equation}

The function values agree exactly at boundaries.

\begin{conclusion}
$R(\mathbf{p})$ is continuous on all of $\Delta^{c-1}$. The $\argmax$ operation introduces no discontinuity because the competing regions have identical function values where they meet. Combined with Lipschitz continuity within each region (Step 2), this establishes global continuity.
\end{conclusion}

\subsubsection*{Step 4: Bounded Range}

\begin{lemma}[Diameter of Probability Simplex]
$\diam(\Delta^{c-1}) = \pi$ under the Fisher-Rao metric.
\end{lemma}

\begin{proof}
The maximum distance occurs between distinct vertices:
\begin{equation}
\dFR(\mathbf{e}_i, \mathbf{e}_j) = 2\arccos\left(\sum_{k=1}^c \sqrt{\delta_{ik} \cdot \delta_{jk}}\right) = 2\arccos(0) = \pi \quad \text{for } i \neq j
\end{equation}
\end{proof}

Since $\dFR(\mathbf{p}, \mathbf{e}_{\hat{j}}) \in [0, \pi]$, the reliability score has range:
\begin{equation}
R \in [\exp(-\lambda \pi), 1]
\end{equation}

\begin{itemize}[leftmargin=*]
\item \textbf{Maximum:} $R = 1$ when $\mathbf{p} = \mathbf{e}_{\hat{j}}$ (perfect confidence in predicted class)
\item \textbf{Minimum:} $R = \exp(-\lambda \pi)$ when $\mathbf{p} = \mathbf{e}_k$ for $k \neq \hat{j}$ (maximum confusion)
\end{itemize}

The effective range is:
\begin{equation}
\Delta R = 1 - \exp(-\lambda \pi)
\end{equation}

For common values of $\lambda$:

\begin{center}
\begin{tabular}{cccc}
\toprule
$\lambda$ & $\exp(-\lambda \pi)$ & Range $\Delta R$ & Interpretation \\
\midrule
0.5 & 0.208 & 0.792 & Low sensitivity \\
1.0 & 0.043 & 0.957 & Moderate sensitivity \\
2.0 & 0.002 & 0.998 & High sensitivity \\
\bottomrule
\end{tabular}
\end{center}

\subsubsection*{Step 5: Concentration via Hoeffding}

\textbf{Hoeffding's Inequality \cite{hoeffding1963probability}:} Let $X_1, \ldots, X_n$ be independent random variables with $X_i \in [a_i, b_i]$ almost surely. Then for $\bar{X} = \frac{1}{n}\sum_{i=1}^n X_i$:
\begin{equation}
\mathbb{P}(|\bar{X} - \mathbb{E}[\bar{X}]| > t) \leq 2\exp\left(-\frac{2n^2t^2}{\sum_{i=1}^n(b_i - a_i)^2}\right)
\end{equation}

\begin{proof}[Application to our setting]
For i.i.d. observations with $R_i \in [e^{-\lambda\pi}, 1]$ having common range $\Delta R = 1 - e^{-\lambda\pi}$, Hoeffding's inequality gives:
\begin{equation}
\mathbb{P}(|\bar{R}_n - \mathbb{E}[R]| > t) \leq 2\exp\left(-\frac{2n^2t^2}{n(\Delta R)^2}\right) = 2\exp\left(-\frac{2nt^2}{(\Delta R)^2}\right)
\end{equation}

For a single random draw ($n=1$), this yields:
\begin{equation}
\boxed{\mathbb{P}(|R - \mathbb{E}[R]| > t) \leq 2\exp\left(-\frac{2t^2}{(1 - e^{-\lambda\pi})^2}\right)}
\end{equation}

This proves part (i) of Theorem~\ref{thm:concentration}.
\end{proof}

\subsubsection*{Step 6: Sub-Gaussian Parameter}

By Hoeffding's lemma, any random variable $X$ with range $[a, b]$ satisfies:
\begin{equation}
\mathbb{E}[e^{s(X - \mathbb{E}[X])}] \leq \exp\left(\frac{s^2(b-a)^2}{8}\right)
\end{equation}

This is the definition of a sub-Gaussian random variable with parameter $\sigma^2 = (b-a)^2/4$.

For $R \in [e^{-\lambda\pi}, 1]$:
\begin{equation}
\boxed{\sigma^2 = \frac{(1 - e^{-\lambda\pi})^2}{4}}
\end{equation}

This proves part (ii) of Theorem~\ref{thm:concentration}.

\begin{remark}[Numerical values]
\begin{itemize}[leftmargin=*]
\item For $\lambda = 1$: $\sigma^2 = (1 - e^{-\pi})^2/4 \approx 0.916/4 \approx 0.229$, giving $\sigma \approx 0.48$
\item For $\lambda = 2$: $\sigma^2 \approx 0.250$, giving $\sigma \approx 0.50$
\end{itemize}
\end{remark}

This proves part (iii): for $\lambda = 1$, the concentration bound becomes:
\begin{equation}
\mathbb{P}(|R - \mathbb{E}[R]| > t) \leq 2\exp(-2.18 \, t^2)
\end{equation}

\subsubsection*{Comparison of Bounds}

Alternative concentration bounds could be derived using Lipschitz properties. We compare approaches:

\begin{center}
\begin{tabular}{lccl}
\toprule
Approach & Variance Bound & For $\lambda=1$ & Method \\
\midrule
Lipschitz + diameter & $\lambda^2\pi^2 = 9.87$ & $2\exp(-0.20 \, t^2)$ & McDiarmid \\
\textbf{Hoeffding on range} & $(1-e^{-\pi})^2 = 0.916$ & $\mathbf{2\exp(-2.18 \, t^2)}$ & Hoeffding \\
\bottomrule
\end{tabular}
\end{center}

The Hoeffding bound using the explicit range $[e^{-\lambda\pi}, 1]$ is approximately $11\times$ tighter than a naive Lipschitz-diameter bound for typical $\lambda \approx 1$.

\begin{remark}[Why the improvement?]
The Lipschitz approach uses diameter $\pi$ as the worst-case distance, but the reliability function decays exponentially with distance, making the effective range $1 - e^{-\lambda\pi} \ll \lambda\pi$ for $\lambda \geq 1$.
\end{remark}

\begin{conclusion}
We have proven all three parts of Theorem~\ref{thm:concentration}:
\begin{enumerate}[label=(\roman*)]
\item The tail bound via Hoeffding's inequality (Step 5)
\item The sub-Gaussian parameter $\sigma^2 = (1-e^{-\lambda\pi})^2/4$ (Step 6)
\item The explicit constant $2.18$ for $\lambda = 1$ (Step 6)
\end{enumerate}

The key technical insight is that the reliability score $R(\mathbf{p})$ is continuous everywhere on $\Delta^{c-1}$ despite the $\argmax$ operation in its definition (Step 3). This continuity, combined with the bounded range from the simplex diameter (Step 4), enables application of Hoeffding's inequality to obtain a tight concentration bound.
\end{conclusion}

\end{document}